\documentclass{article}
\usepackage{etoolbox}

\usepackage{graphicx}

\usepackage[parfill]{parskip}
\usepackage{amssymb,amsmath,amsthm,bm}
\usepackage{mathtools}
\usepackage[dvipsnames,table]{xcolor}
\usepackage{thmtools}
\usepackage{adjustbox}
\usepackage{arydshln}
\usepackage{colortbl}
\usepackage{booktabs}
\usepackage{multirow}
\usepackage{natbib}
\usepackage{authblk}
\PassOptionsToPackage{round, compress}{natbib}
            
\usepackage[parfill]{parskip}

\newlength{\defbaselineskip}
\RequirePackage[T1]{fontenc}

\renewcommand{\sfdefault}{phv}

\usepackage{bm}
\usepackage[inline]{enumitem}
\usepackage{booktabs}
\usepackage{subcaption}
\usepackage{float}
\usepackage{multirow}
\usepackage{wrapfig,lipsum}
\usepackage{algorithm,algorithmicx,algpseudocode}
\usepackage{nicematrix}
\usepackage{nicefrac}
\usepackage{tikz}
\usetikzlibrary{positioning, arrows.meta, shapes.geometric, calc, backgrounds, fit}

\usepackage[colorlinks=true,linkcolor=blue,citecolor=blue,urlcolor=black]{hyperref}
\usepackage[capitalise,noabbrev]{cleveref}

\usepackage{import}

\usepackage[utf8]{inputenc}
\usepackage[T1]{fontenc}
\usepackage{csquotes}
\usepackage[osf,sc]{mathpazo}
\RequirePackage[scaled=0.90]{helvet}
\RequirePackage[scaled=0.85]{beramono}
\RequirePackage{textcomp}

\usepackage{siunitx}
\usepackage[normalem]{ulem}
\usepackage{microtype}
\robustify\bfseries
\robustify\uline
\newrobustcmd\Bf{\DeclareFontSeriesDefault[rm]{bf}{b}\bfseries}

\definecolor{addgreen}{rgb}{0.9,0.0,0.0}

\newcommand{\1}{\bm{1}}

\newcommand{\defeq}{\coloneqq}

\newcommand{\norm}[1]{\lVert {#1} \rVert}

\newcommand{\diag}{\mathop{\mathrm{diag}}}

\newcommand{\lrangle}[1]{\langle#1\rangle}

\newcommand{\lrbp}[1]{\!\big(#1\big)\!}

\newcommand{\normf}[1]{\left\| #1 \right\|_{\mathrm{F}}}

\newcommand{\sphere}[1]{\gS^{#1}}

\def\eqref#1{equation~\ref{#1}}

\def\1#1{\bm{1}_{#1}}

\def\eps{{\epsilon}}

\def\vone{{\bm{1}}}

\def\va{{\bm{a}}}
\def\vb{{\bm{b}}}
\def\vc{{\bm{c}}}

\def\ve{{\bm{e}}}

\def\vk{{\bm{k}}}

\def\vo{{\bm{o}}}

\def\vq{{\bm{q}}}

\def\vt{{\bm{t}}}
\def\vu{{\bm{u}}}
\def\vv{{\bm{v}}}
\def\vw{{\bm{w}}}
\def\vx{{\bm{x}}}

\def\vy{{\bm{y}}}
\def\vz{{\bm{z}}}

\def\vpsi{{\bm{\psi}}}
\def\vgamma{{\bm{\gamma}}}
\def\hvq{\hat{\bm{\vq}}}
\def\hvk{\hat{\bm{\vk}}}

\def\mA{{\bm{A}}}
\def\mB{{\bm{B}}}

\def\mD{{\bm{D}}}

\def\mI{{\bm{I}}}

\def\mK{{\bm{K}}}
\def\mL{{\bm{L}}}
\def\mM{{\bm{M}}}
\def\mN{{\bm{N}}}
\def\mO{{\bm{O}}}
\def\mP{{\bm{P}}}
\def\mQ{{\bm{Q}}}

\def\mS{{\bm{S}}}
\def\mT{{\bm{T}}}
\def\mU{{\bm{U}}}
\def\mV{{\bm{V}}}

\def\mZ{{\bm{Z}}}

\DeclareMathAlphabet{\mathsfit}{\encodingdefault}{\sfdefault}{m}{sl}
\SetMathAlphabet{\mathsfit}{bold}{\encodingdefault}{\sfdefault}{bx}{n}

\def\gH{{\mathcal{H}}}

\def\gK{{\mathcal{K}}}
\def\gL{{\mathcal{L}}}

\def\gN{{\mathcal{N}}}
\def\gO{{\mathcal{O}}}

\def\gS{{\mathcal{S}}}
\def\gT{{\mathcal{T}}}

\def\sS{{\mathbb{S}}}

\newcommand{\E}{\mathbb{E}}

\newcommand{\R}{\mathbb{R}}

\DeclareMathOperator{\Tr}{Tr}

\newcommand{\dd}{\mathrm{d}}

\newcommand{\seql}{T}

\newtheorem{theorem}{Theorem}
\newtheorem{proposition}[theorem]{Proposition}
\newtheorem{lemma}[theorem]{Lemma}

\theoremstyle{definition}

\newtheorem{assumption}{Assumption}
\crefname{assumption}{Assumption}{Assumptions}
\Crefname{assumption}{Assumption}{Assumptions}

\usetikzlibrary{positioning, arrows.meta, shapes.geometric, calc, backgrounds, fit}

\title{Kernelized Linear Attention:\\Breaking the Capacity Wall with Symmetric Cones}
\author[ ]{Ayoub Ghriss\,\thanks{Corresponding author.}}
\author[ ]{Sourav Chakraborty}

\affil[ ]{\normalsize Department of Computer Science, University of Colorado Boulder}
\affil[ ]{\texttt{\{ayoub.ghriss, sourav.chakraborty\}@colorado.edu}}

\date{}

\hypersetup{
  pdftitle={Kernelized Linear Attention: Breaking the Capacity Wall with Symmetric Cones},
  pdfauthor={Ayoub Ghriss and Sourav Chakraborty},
  pdfsubject={Machine learning, linear attention, and associative recall},
  pdfkeywords={linear attention, associative recall, spherical packing, symmetric cones, positive semidefinite features, GPU kernels}
}

\begin{document}
\maketitle


\begin{abstract}
Linear attention promises constant-time recurrent inference but degrades sharply on associative recall.
We formulate attention recall as a spherical-packing problem and introduce Kernelized Linear Attention Activations (KATA),
a framework whose feature maps are derived from first principles by certifying nonnegative attention weights through a self-dual homogeneous cone.
Building on this observation, we show that rank-one positive semi-definite (PSD) features offer a favorable capacity--interference tradeoff.
KATA recovers a parameter-free convex output gate and characterizes associative capacity through the Welch interference floor.
For tolerances above this floor, KATA enlarges the state without adding parameters and admits spherical codes with exponentially many keys in the projection dimension.
We implement KATA as fused Triton kernels\footnote{\href{https://github.com/ayghri/kata}{\color{blue}\texttt{https://github.com/ayghri/kata}}} at two operating points: a flash-attention-style forward up to ${\sim}1.6\times$ FlashAttention-2 throughput, and an exact $O(T)$ chunked-state form that reaches ${\sim}11\times$ FlashAttention-2 forward throughput at $131$k tokens.
An associative scan of the first-order feature lowers the inter-chunk recurrence depth to $O(\log(T/C))$ for chunk size $C$ and averages ${\sim}2.4\times$ the throughput of a matched sequential linear-attention baseline.
On long-range MQAR and repeated-key overwrite, several KATA variants outperform Gated DeltaNet, with parameter counts and state sizes reported alongside accuracy.
Induction preserves near-perfect recall, while kernel benchmarks show that the maps can be implemented efficiently.
KATA retains $0.985$ MQAR at a $16\times$ out-of-distribution length, approaching the softmax with roughly one quarter of the KV-cache entries.
Experiments on 340M-parameter LLMs reveal a feature-dependent fluency trade-off and clarify how positional embeddings, delta rules, and decay gates interact with feature geometry.
\end{abstract}

\section{Introduction} \label{sec:intro}
Transformer architectures~\citep{vaswani17attention} owe their dominance to softmax attention, whose exponential kernel
supports near-perfect associative recall. The price is the $\gO(\seql^2)$ training cost and the
$\gO(\seql)$ KV cache at sequence length $\seql$, which has motivated a sustained body of
linear-time alternatives and fixed-size recurrent
states~\citep{katharopoulos20transformers,choromanski21performers,peng21rfa,sun23retentive,gu24mamba,dao24mamba2}.
Linear attention replaces the infinite feature map of $\exp(\lrangle{\vq,\vk})$ with a finite
$\psi:\R^d\to\R^{n_\psi}$ and exposes the recurrent update $\mS_t=\mS_{t-1}+\psi(\vk_t)\vv_t^\top$. The
price of constant state is \emph{memory collision}: many key--value bindings must share a single
fixed state, and retrieval degrades sharply on associative-recall tasks at long
contexts~\citep{schlag21linear,jelassi24repeat,du25mom}.

Recent gated linear RNNs, including DeltaNet~\citep{yang24deltanet},
Mamba2~\citep{dao24mamba2}, and Gated DeltaNet~\citep{yang25gateddeltanet}, close part of the gap
with learned forget gates, unnormalized linear writes, and content-addressed erase operations. Yet
the geometry driving their feature maps remains opaque: what limits recall capacity at a fixed
state size, and which post-Transformer designs operate below that limit? Must feature maps be
nonnegative? Which finite maps preserve the isotropy of the softmax? Without principled answers,
new linear attention architectures sit one ablation away from the existing menagerie of variants.

\paragraph{Central finding.} Memory collision is, at heart, the saturation of a spherical-packing
problem. Once nonnegative attention weights are certified by a self-dual homogeneous cone and the
latent geometry is required to be isometrically invariant, the Koecher--Vinberg
classification~\citep{koecher57positivitatsbereiche,vinberg63homogeneous,faraut94analysis}
organizes the admissible feature geometries into irreducible factors. We study the ordinary real
families: the positive orthant, the Lorentz cone, and the symmetric positive semidefinite (PSD) cone.
Their feature maps form the basis of our \textbf{KATA} (Kernelized Linear Attention Activations)
framework. The denominator of normalized linear attention emerges natively as a parameter-free,
token-conditioned convex output gate, so part of the structure typically introduced through explicit
gating is already inherent to the normalized formulation. The choice of cone then determines the
interference geometry and, through it, the associative capacity.

To compare these geometries, we develop a power signal-to-noise ratio (pSNR) analysis based on
the optimal interference attainable within each cone. Applying the same analysis to an exponential
readout also yields a capacity law for softmax at arbitrary temperature. We support these predictions
with efficient Triton implementations across a wide range of context lengths and with pretraining
experiments on 340M-parameter language models. Together, the theory and experiments clarify how
feature geometry interacts with positional embeddings, gating, the delta rule, and softmax attention.

\paragraph{Working on the unit sphere.} Throughout, we analyze query and key directions on the
unit sphere. Explicit $\ell_2$ normalization enforces this constraint directly. Under RMSNorm, the
per-channel gains can instead be absorbed into a learned diagonal bilinear form and an
inverse-temperature scale, leaving normalized directions in a lifted representation
(\Cref{app:rms-temperature}). This setting covers the query--key normalization used in
Qwen3~\citep{yang25qwen3}, Gemma~3~\citep{kamath25gemma3}, and
MiniMax-class~\citep{chen26minimax} models, as well as DeltaNet~\citep{yang24deltanet} and Gated
DeltaNet~\citep{yang25gateddeltanet}. The capacity bounds therefore describe the normalized address
geometry used by these architectures.

\paragraph{Contributions.}
\begin{enumerate}
  \item \textbf{A cone classification for nonnegative linear attention.} We characterize the
        ordinary real symmetric-cone feature geometries considered here under nonnegative attention
        weights, isometric invariance, and homogeneous latent geometry.
  \item \textbf{A capacity theory for associative recall.} Casting recall as discrete spherical
        packing, we prove the Lorentz Rankin wall (\Cref{thm:cap-lor}) and dimension-dependent packing bounds for PSD rank-one rays
        (\Cref{thm:cap-spd}), give explicit finite dictionaries from mutually unbiased
        bases and DeVore frames, and derive an optimal-packing pSNR ladder whose geometric argument
        also applies to temperature-scaled softmax attention.
  \item \textbf{Hardware-aligned kernels.} We release fused Triton kernels for the gate-free,
        delta-rule-free KATA-M$g$ recurrence at two operating points. Its quadratic
        $\gO(\seql^2)$ forward reaches up to ${\sim}1.6\times$ FlashAttention-2
        throughput~\citep{dao23fa2} and remains competitive on the training step. In the
        linear-state $\gO(\seql)$ form, KATA-M2 reaches parity with its quadratic counterpart near
        $16$k tokens, is faster by $32$k, and reaches ${\sim}11\times$ FlashAttention-2's forward
        throughput at $131$k; KATA-M1 is faster by $128$k. For the linear feature, an associative scan
        lowers the inter-chunk depth to $\gO(\log N_C)$ for
        $N_C\defeq\lceil\seql/C\rceil$ chunks and averages ${\sim}2.4\times$ the throughput
        of a matched sequential linear-attention chunk.
  \item \textbf{Experimental evidence across scales.} In the Zoology
        playground~\citep{arora23zoology}, we evaluate multi-query associative recall, repeated-key
        overwrite, and induction under matched training configurations, reporting both parameter counts and state sizes. KATA-M1 retains
        $0.985$ MQAR accuracy at $16\times$ the training context with roughly one quarter of
        softmax's KV-cache entries at that length. We then pretrain matched 340M-parameter
        language models for $15$B tokens and evaluate both standard language-model benchmarks
        and in-context recall. The PSD variants remain broadly comparable on zero-shot accuracy
        while showing a meaningful perplexity spread, and they retain much more high-entropy
        needle signal than Gated DeltaNet, whose UUID recall falls to $0.004$.
\end{enumerate}

\section{Reproducing Kernel Hilbert Spaces for Linear Attention} \label{sec:rkhs}
Linear attention~\citep{shen21efficient,katharopoulos20transformers} rests on a single algebraic
substitution: replacing the exponential kernel of softmax with an inner product of finite features.
This substitution allows the quadratic mixing of tokens to collapse into a recurrence on a
fixed-size state. We make this substitution explicit and identify the geometric assumptions it imposes on the
feature map. Constraining the feature image to a symmetric cone then resolves two design questions
typically addressed through ad hoc heuristics: which finite maps are admissible and how normalization
produces an output gate.

Throughout, sequences have length $\seql$, with
$\vq_t,\vk_t\in\R^d$ and $\vv_t\in\R^{d_v}$ denoting the query, key, and value at token $t$.
Their stacked tensors are $\mQ,\mK\in\R^{\seql\times d}$ and
$\mV\in\R^{\seql\times d_v}$. We write $\psi:\R^d\!\to\!\R^{n_\psi}$ for the feature map and
$\Psi(\mK)$ for its row-wise application. Causal softmax attention computes, per head:
\begin{equation}
  \vz_t \;=\; \frac{\sum_{i=1}^t \exp(\lrangle{\vq_t,\vk_i})\,\vv_i^\top}{
    \sum_{i=1}^t \exp(\lrangle{\vq_t,\vk_i})}.
  \label{eq:softmax-att}
\end{equation}
Substituting the softmax kernel $\exp(\lrangle{\vq,\vk})$ with a finite inner product
$\lrangle{\psi(\vq),\psi(\vk)}$ and using bilinearity yields the linearized form:
\begin{equation}
  \vz_t \;=\; \frac{\psi(\vq_t)^\top \mS_t}{\psi(\vq_t)^\top \mZ_t}, \quad
  \mS_t = \sum_{i=1}^t \psi(\vk_i)\vv_i^\top, \quad
  \mZ_t = \sum_{i=1}^t \psi(\vk_i),
  \label{eq:linatt}
\end{equation}
with state pair $(\mS_t,\mZ_t)$ of size $n_\psi(d_v+1)$ for all $t$, and additive recurrence
$\mS_t=\mS_{t-1}+\psi(\vk_t)\vv_t^\top$, $\mZ_t=\mZ_{t-1}+\psi(\vk_t)$. The normalized readout is defined whenever its denominator
$\psi(\vq_t)^\top\mZ_t$ is positive.

\subsection{Symmetric cones}
Let $\gK\subset\R^{n_\psi}$ be a closed convex cone. Its dual cone is defined as $\gK^* = \{\vy\in\R^{n_\psi} :
  \lrangle{\vx,\vy}\ge 0\;\forall\vx\in\gK\}$ under the Euclidean inner product. $\gK$ is
\emph{pointed} if $\gK\cap(-\gK)=\{0\}$, \emph{full-dimensional} if
$\operatorname{int}(\gK)\ne\emptyset$, \emph{self-dual} if $\gK=\gK^*$, and \emph{homogeneous} if
its automorphism group acts transitively on $\operatorname{int}(\gK)$. $\gK$ is a \emph{symmetric}
cone if it is self-dual and homogeneous. Throughout this paper, we work under the following three
assumptions:
\begin{assumption}[Self-dual nonnegative geometry]
  \label{ass:pos}
  $\psi(\R^d)\subset\gK=\gK^*$. Consequently,
  $\lrangle{\psi(\vx),\psi(\vy)}\ge 0$ for all $\vx,\vy$.
\end{assumption}

\begin{assumption}[Isometric invariance]
  \label{ass:iso}
  For every $\mU\in\gO(\R^d)$ there exists an inner-product-preserving map $\gT_\mU$ on $\gK$ such
  that $\psi(\mU\vx)=\gT_\mU\psi(\vx)$.
\end{assumption}

\begin{assumption}[Homogeneity]
  \label{ass:homogene}
  $\mathrm{Aut}(\gK)$ acts transitively on $\operatorname{int}(\gK)$.
\end{assumption}

\subsection{The denominator is a parameter-free convex gate}

Earlier work on linear attention enforced the nonnegativity in \Cref{ass:pos} through the feature
map $\psi(\vx)=\operatorname{ELU}(\vx)+1$~\citep{katharopoulos20transformers}, whereas Efficient
Attention~\citep{shen21efficient} omitted the denominator $\psi(\vq_t)^\top\mZ_t$. Later recurrent
architectures instead introduced explicit recurrence-level gates, including the hardware-friendly
vector gate of Gated Linear Attention (GLA)~\citep{yang24gla} and the scalar decay of
Mamba2~\citep{dao24mamba2}.

Under~\Cref{ass:pos}, we show that the discarded denominator already supplies a token-conditioned
output gate.
\begin{proposition}[Output-level gating]
  \label{prop:output-gate}
  Fix $\vq_t$ and write:
  \[
    D_{t-1}=\psi(\vq_t)^\top\mZ_{t-1},
    \qquad
    c_t=\psi(\vq_t)^\top\psi(\vk_t),
    \qquad
    D_t=D_{t-1}+c_t.
  \]
  Suppose $D_t>0$. If $D_{t-1}>0$, define
  $\bar\vz_{t|t-1}=\psi(\vq_t)^\top\mS_{t-1}/D_{t-1}$. Then:
  \begin{align}
    \vz_t
      &=\alpha_t(\vq_t)\,\bar\vz_{t|t-1}
        +\beta_t(\vq_t)\,\vv_t^\top, \\
    \alpha_t
      &=\frac{D_{t-1}}{D_t},
    \qquad
    \beta_t
      =\frac{c_t}{D_t},
    \label{eq:output-gate}
  \end{align}
  with $\alpha_t,\beta_t\ge0$ and $\alpha_t+\beta_t=1$. If $D_{t-1}=0$, including
  $t=1$, then $\vz_t=\vv_t^\top$ and the same weights reduce to
  $(\alpha_t,\beta_t)=(0,1)$.
\end{proposition}
The normalized readout is therefore a convex interpolation between the previous readout and the
current value, with query-dependent weights and no additional learned parameters. The gate acts at
readout, leaving the state update commutative and additive. Gated DeltaNet controls memory through
the rank-one state transition~\citep{yang25gateddeltanet}:
\[
  \mS_t
  =g_t(\mI-\beta_t\vk_t\vk_t^\top)\mS_{t-1}
   +\beta_t\vk_t\vv_t^\top,
\]
which uses extra projections and produces a noncommutative transition.

\subsection{Koecher--Vinberg classification}
\begin{figure*}[t]
  \centering
  \tiny
  \begin{subfigure}[b]{0.32\textwidth}\centering
    \includegraphics[width=0.6\textwidth]{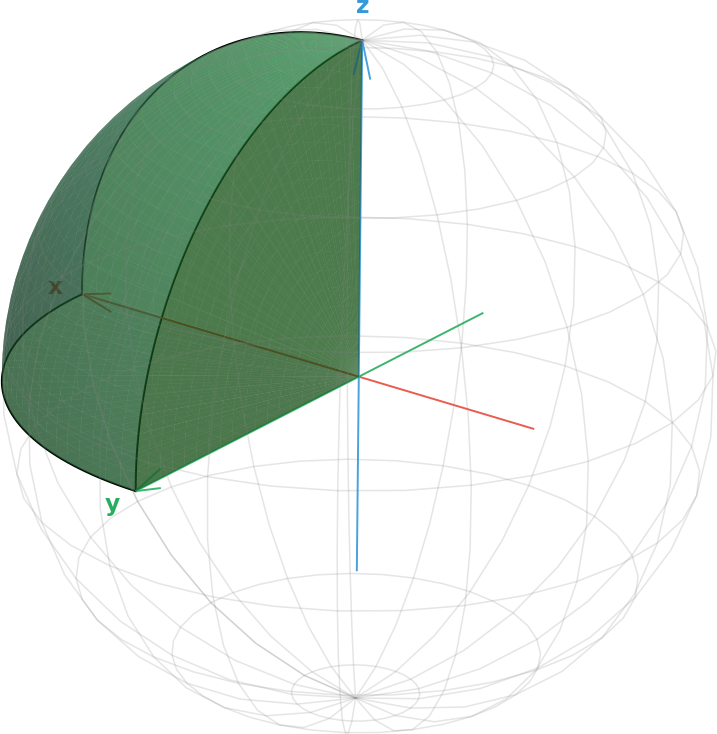}
    \caption{Positive orthant $\R^3_+$.}\label{fig:image1}
  \end{subfigure}\hfill
  \begin{subfigure}[b]{0.32\textwidth}\centering
    \includegraphics[width=0.6\textwidth]{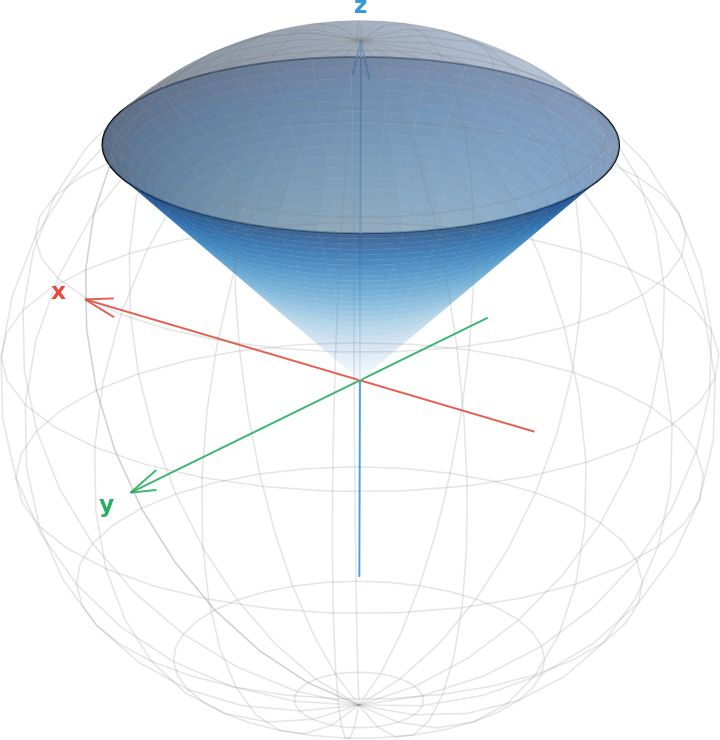}
    \caption{Lorentz cone $\gL^3_+$.}\label{fig:image2}
  \end{subfigure}\hfill
  \begin{subfigure}[b]{0.32\textwidth}\centering
    \includegraphics[width=0.6\textwidth]{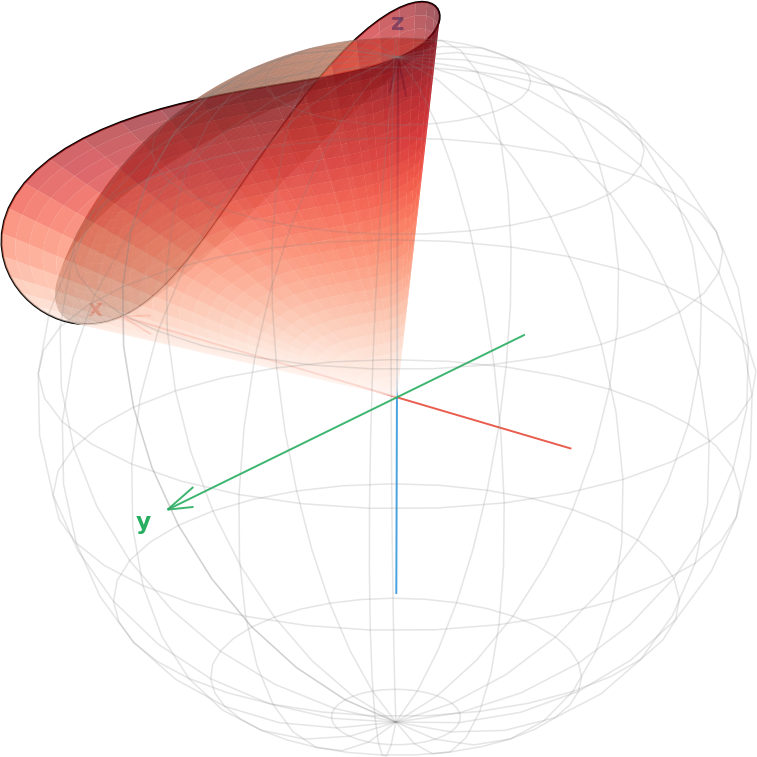}
    \caption{PSD cone $\sS^2_{+}$.}\label{fig:image3}
  \end{subfigure}
  \caption{\textbf{Ordinary real symmetric-cone geometries in $\R^3$.}
    The PSD cone $\sS^2_{+}$ is linearly isomorphic to the Lorentz cone $\gL^3_+$.}
  \label{fig:two_images}
\end{figure*}

Isometric invariance from~\Cref{ass:iso} promotes the softmax identity
$\exp(\lrangle{\mU\vx,\mU\vy})=\exp(\lrangle{\vx,\vy})$ to the kernel space. On its own, it does
not single out $\gK$ but guarantees that $\psi$ is $\gT_\mU$-equivariant: the input-space isometry
$\mU\in\gO(\R^d)$ and the cone-space isometry $\gT_\mU\in\gO(\R^{n_\psi})|_\gK$ intertwine through $\psi$.

Homogeneity in~\Cref{ass:homogene} means that any interior point can be mapped to any other by a
cone automorphism. Thus the latent geometry has no privileged interior region or structural
bottleneck, a useful symmetry when learning a map into $\gK$.

\Cref{ass:pos,ass:homogene} place $\gK$ within the scope of the classical
Koecher--Vinberg theorem, while \Cref{ass:iso} constrains how input rotations act on its factors.
The full classification also contains complex Hermitian, quaternionic Hermitian, and exceptional
factors. We restrict this work to the ordinary real factors listed below.
\begin{theorem}[Koecher--Vinberg decomposition~\citep{koecher57positivitatsbereiche,vinberg63homogeneous,faraut94analysis}]
  \label{thm:kv}
  Let $\gK\subset\R^{n_\psi}$ be a pointed, full-dimensional closed convex cone. If $\gK$ is self-dual
  and homogeneous, then it is a symmetric cone and decomposes uniquely, up to isometry and factor
  order, as a Cartesian product of irreducible symmetric cones. The irreducible factors studied here
  are:
  \[
    \begin{cases}
      \textnormal{\textsc{(O)}} & \R_+,                                \\
      \textnormal{\textsc{(L)}} & \gL^m_+
      = \{(\vy,t)\in\R^{m-1}\times\R : \norm{\vy}_2 \le t\},\  m\ge 3, \\
      \textnormal{\textsc{(S)}} & \sS^m_+
      = \{\mA\in\R^{m\times m} : \mA\succeq 0 \text{ and } \mA=\mA^{\top}\},\ m\ge 2.
    \end{cases}
  \]
\end{theorem}

\paragraph{Classical feature maps}
We can map raw projected keys and queries to the positive orthant via
$\mathrm{ReLU}(\vx)+\eps$ for some $\eps>0$, or via
$\mathrm{ELU}(\vx)+1$ as in linear attention~\citep{katharopoulos20transformers}. For the Lorentz
cone, $\psi^{\mathrm L}(\vy)=(\vy,\norm{\vy}_2)$ lies on the boundary of $\gL^{d+1}_+$ for
$\vy\ne0$, and Cauchy--Schwarz inequality yields:
\[
  \lrangle{\psi^{\mathrm L}(\vx),\psi^{\mathrm L}(\vy)}
  = \lrangle{\vx,\vy}+\norm{\vx}_2\norm{\vy}_2 \ge 0.
\]
If $\norm{\vx}_2=1$, then
$\norm{\psi^{\mathrm L}(\vx)}_2^2
 =\lrangle{\psi^{\mathrm L}(\vx),\psi^{\mathrm L}(\vx)}
 =2\norm{\vx}_2^2=2$.
Thus $2^{-1/2}\psi^{\mathrm L}(\vx)$ lies on the unit sphere.\par

For the PSD cone $\sS^m_+$, a classical choice is the LDL parameterization, which represents any
$\mA \in \sS^m_+$ as $\mA = \mL_A \mD_A \mL_A^\top$, where $\mD_A$ is diagonal and $\mL_A$ is unit
lower triangular. Then:
\[
\lrangle{\mA,\mB} = \Tr\!\left(\mL_A \mD_A \mL_A^\top \mL_B \mD_B
  \mL_B^\top\right) = \normf{\mD_B^{1/2}\mL_B^\top \mL_A \mD_A^{1/2}}^2.
\] 

These canonical maps establish nonnegative kernels for each real factor considered here. We next turn
from admissible feature geometry to the number of key addresses that each cone can separate.

\section{Capacity of Symmetric Cones} \label{sec:expressivity}
The cone classification identifies admissible nonnegative feature geometries, but not how many key
addresses a fixed feature dimension can distinguish. Let $\gK\subset\R^d$ be a cone and place $T$
unit-norm feature vectors on $\gK\cap\sphere{d-1}$. We quantify their worst collision by the
\emph{interference}:
\begin{equation}
  \mu(\gK,T)
  \defeq
  \min_{\{\bm{\vpsi}_1,\dots,\bm{\vpsi}_T\}\subset\gK\cap\sphere{d-1}}
  \max_{i\neq j}
  \big|\lrangle{\bm{\vpsi}_i,\bm{\vpsi}_j}\big|,
  \label{eq:mu-def}
\end{equation}
and define the \emph{$\eps$-capacity}:
\begin{equation}
  \mathsf{C}_{\eps}(\gK)
  \defeq
  \sup\{T\in\mathbb{N}:\mu(\gK,T)\le\eps\}.
  \label{eq:eps-capacity}
\end{equation}
This is precisely the mutual coherence of the feature dictionary~\citep{donoho03optimally}. It
isolates the geometry of distinguishable key addresses before specifying values or retrieval
dynamics.

\begin{theorem}[Welch bound~\citep{welch74lower}]
  \label{thm:welchbound}
  The $d$ coordinate rays of $\R_+^d$ are pairwise orthogonal, so zero interference is attainable
  for $T\le d$. For any $T>d$ unit vectors in $\R^d$:
  \begin{equation}
    \mu(\R^d,T)^2\ge \frac{T-d}{d(T-1)}.
  \end{equation}
  Consequently, if $\mu\le\eps<1/\sqrt d$, then
  $T\le d(1-\eps^2)/(1-d\eps^2)$.
\end{theorem}
For $T>d$, equality in the Welch inequality holds precisely when the vectors form an equiangular tight frame
(ETF): tightness minimizes the average squared off-diagonal Gram entry, and equiangularity makes
every off-diagonal modulus attain that average. The bound is therefore generally unattained. Real
ETFs exist only for special pairs $(d,T)$ and satisfy $T\le d(d+1)/2$~\citep{fickus16etf}. Over
$\mathbb C^d$, the corresponding ceiling is $d^2$, although the representation uses $2d$ real
scalars. We restrict the present analysis to real cones and leave complex extensions to future work.

The orthant inherits the ordinary Welch floor after its $d$ coordinate rays. For the Lorentz cone
$\gL_+^d$, a normalized extreme ray has the form
$\psi=2^{-1/2}(\vy,1)$ with $\vy\in\sphere{d-2}$, and its kernel is
$F_{\mathrm{Lor}}(z)=(1+z)/2$. For $0\le\eps<1/2$, requiring feature interference at most $\eps$ forces
$\lrangle{\vy_i,\vy_j}\le-(1-2\eps)$. Rankin's bound therefore gives:
\[
  \mathsf{C}_{\eps}(\gL_+^d)
  \le \min\!\left\{d,\,1+\frac{1}{1-2\eps}\right\}.
\]
The formal statement and proof appear as \Cref{thm:cap-lor} in \Cref{app:packing}.

\subsection{Rank-one packing into the PSD cone}
\label{ssec:spd-welch}

The LDL construction in \Cref{sec:rkhs} represents an arbitrary matrix in $\sS_+^p$ using:
\begin{equation}
  n=\frac{p(p+1)}{2}
  \label{eq:psd-ambient-dim}
\end{equation}
free coordinates. Directly parameterizing an arbitrary PSD atom would therefore require a learned
query/key projection of width $n$. Packing does not require covering the entire cone: we can project
once to $\vu\in\sphere{p-1}$ and apply the deterministic rank-one lift:
\[
  \psi^{\mathrm S}(\vu)=\vu\vu^\top\succeq0,
  \qquad
  \norm{\vu\vu^\top}_F=1,
  \qquad
  \lrangle{\vu\vu^\top,\vv\vv^\top}_F=(\vu^\top\vv)^2.
\]
The learned projection emits only $p$ coordinates, while the lifted feature still occupies the
$n$-dimensional symmetric ambient space. \Cref{prop:spd-rank-one} shows that, among unit-Frobenius
PSD atoms, rank-one matrices uniquely minimize the expected overlap under a Haar-random relative
eigenbasis. This average-case optimality motivates the rank-one family; the following theorem gives
its worst-case packing guarantee.

\begin{theorem}[PSD rank-one packing]
  \label{thm:cap-spd}
  Let $p\ge2$, $n=p(p+1)/2$, and $0\le\eps<1$. Identify $\sS_+^p$ with its
  $n$-dimensional Euclidean ambient space under the Frobenius inner product. The unit-Frobenius representatives of its extreme rays are rank-one matrices
  $\vu\vu^\top$ with $\vu\in\sphere{p-1}$, and their pairwise inner products satisfy:
  \[
    \lrangle{\vu_i\vu_i^\top,\vu_j\vu_j^\top}_F
    =\lrangle{\vu_i,\vu_j}^2.
  \]
  Hence PSD interference at most $\eps$ is equivalent to
  $|\lrangle{\vu_i,\vu_j}|\le\sqrt{\eps}$, and a maximal greedy spherical-cap construction yields:
  \begin{equation}
    \mathsf{C}_{\eps}(\sS_+^p)
    \ge \frac{1}{2}(1-\eps)^{-(p-1)/2}.
  \end{equation}
\end{theorem}

Applying \Cref{thm:welchbound} to the raw vectors shows that PSD interference $\eps<1/p$ forces
$T\le p(1-\eps)/(1-p\eps)$. Rank-one PSD features therefore do not provide exponential capacity
below the Welch floor. For any fixed tolerance $\eps\in(0,1)$, once $p>1/\eps$ the spherical-cap
lower bound in \Cref{thm:cap-spd} is exponential in $p$.

Squaring the raw inner product replaces the mutually obtuse constraint induced by the Lorentz lift
with an unsigned near-orthogonal spherical code. The same $p$-dimensional projection can generate
higher even-order features:
\[
  \psi_{2r}(\vu)=\vu^{\otimes 2r},
  \qquad
  \lrangle{\psi_{2r}(\vu),\psi_{2r}(\vv)}
    =(\vu^\top\vv)^{2r}\ge 0,
\]
whose symmetric ambient dimension is $\binom{p+2r-1}{2r}$. We use the first nontrivial order,
$r=1$: as \Cref{sec:hardware} shows, the quadratic expanded state already dominates the hardware budget.

\subsection{Constructive packing} \label{ssec:constructive}
\begin{wraptable}{r}{0.4\textwidth}
  \centering
  \footnotesize
  \caption{\textbf{Constructive packing requirements.} Ambient dimension required for
  $T=10^5$ under two interference targets.}
  \label{tab:constructive-capacity}
  \begin{tabular}{@{}lll@{}}
    \toprule
    Method  & $\mu\le0.05$       & $\mu\le0.20$ \\
    \midrule
    Welch  & $d\ge399$          & $d\ge25$          \\
    Random & $d\approx18{,}421$ & $d\approx1{,}151$ \\
    MUB    & $d=802$ ($s=401$)  & $d=634$ ($s=317$) \\
    DeVore & $d=2{,}209$ ($s=47$) & $d=361$ ($s=19$) \\
    \bottomrule
  \end{tabular}
\end{wraptable}

We give two deterministic finite dictionaries with explicit interference guarantees. The first realifies complete sets of complex mutually unbiased bases (MUBs), producing a dense dictionary whose size is quadratic in its real ambient dimension.

\begin{restatable}[MUB-derived real dictionary]{theorem}{mubconstruct}
  \label{thm:mub}
  For any prime power $s=\ell^k$, realification of a complete set of complex MUBs gives
  $T=s(s+1)$ unit vectors in $\R^d$, where $d=2s$, with:
  \[
    \mu(\R^d,T)\le\frac{1}{\sqrt{s}}=\sqrt{\frac{2}{d}}.
  \]
\end{restatable}

This construction is useful when $T=\Theta(d^2)$ is sufficient and a uniform deterministic worst-case guarantee matters more than sparsity. The prime-power condition restricts the available dimensions. The full construction is given in \Cref{app:mub_construct}.
When higher interference is tolerated, DeVore's polynomial construction provides a larger and sparser dictionary.

\begin{restatable}[DeVore polynomial construction]{theorem}{devoreconstruct}
  \label{thm:devore}
  For any prime power $s$ and integer $0\le r<s$, there is a sparse deterministic construction of
  $T=s^{r+1}$ unit vectors in $\R^d$, where $d=s^2$, with maximum interference
  $\mu\le r/s$.
\end{restatable}
This construction is preferable when sparse storage or fast overlap computation matters more than the tighter interference of the MUB-derived dictionary. The proof is in \Cref{app:devore-construct}.

The Johnson--Lindenstrauss (JL) lemma gives a related randomized asymptotic statement: $T$ prescribed points can be embedded into
$d=\gO(\eta^{-2}\log(T/\delta))$ dimensions while preserving all pairwise distances within
$1\pm\eta$ with failure probability at most $\delta$
~\citep{johnson84extensions,achlioptas03database}. The same concentration argument yields random
spherical dictionaries with interference at most $\mu$ in
$d=\gO(\mu^{-2}(\log T+\log(1/\delta)))$ dimensions, or equivalently an achievable
$T=\exp(\Omega(\mu^2d))$ at fixed success probability.
These asymptotic random-projection guarantees are often loose at the dimensions of individual
attention heads. Deterministic constructions give exact finite dictionaries and worst-case
interference guarantees in the regimes where they apply. The bounds and finite-size comparison are
collected in \Cref{app:geom-bounds,tab:constructive-capacity}. These constructions illustrate the
available finite regimes, and better real spherical packings may exist. Determining which packings are best suited to attention heads is left to future work.

\subsection{Signal-to-noise analysis}
\label{ssec:snr}

Packing controls the score margin between a matched key and its distractors. Whether that margin is
sufficient for retrieval depends on the readout. To isolate key interference, hold the attention
coefficients fixed and take independent isotropic values
$\vv_j\sim\gN(0,\mI_{d_v}/d_v)$. Then $\mathbb{E}\norm{\vv_j}^2=1$, the cross terms vanish in expectation, and for
$\vo(\vq)=\sum_j a_j(\vq)\vv_j$ with matched index $\star$:
\begin{equation}
  \mathrm{pSNR}(\vq)
    \defeq
    \frac{\mathbb{E}\norm{a_\star\vv_\star}^2}
    {\mathbb{E}\norm{\sum_{j\ne\star}a_j\vv_j}^2}
    =\frac{\lvert a_\star\rvert^2}
    {\sum_{j\ne\star}\lvert a_j\rvert^2}.
  \label{eq:psnr-definition}
\end{equation}
We call the readout \emph{ideally retrievable} when $\mathrm{pSNR}>1$, so the matched value has more
power than the aggregate distractors. This criterion gives the expected power ratio under the isotropic-value model. Exact recall also
depends on the realized values. Deterministic orthogonal unit values give the same identity
when $d_v\ge T$.

\begin{restatable}[Idealized Welch-scale retrieval]{theorem}{optimalsnr}
  \label{thm:snr}
  Let $p\ge2$ and $T>p$. Suppose the $T-1$ distractors have squared correlations:
  \[
    \eps^\star=\frac{T-p}{p(T-1)}
  \]
  with the matched query. Then:
  \begin{align}
    \mathrm{pSNR}_{\mathrm{Lin}}
      &=\frac{1}{(T-1)\eps^\star}
       =\frac{p}{T-p}, \\
    \mathrm{pSNR}_{\mathrm{PSD}}
      &=\frac{1}{(T-1)(\eps^\star)^2}
       =\frac{p^2(T-1)}{(T-p)^2}.
  \end{align}
\end{restatable}

The criterion $\mathrm{pSNR}>1$ gives linear retrieval only for $T<2p$, hence capacity
$\Theta(p)$. The PSD map squares distractor interference and extends the idealized threshold to
$\Theta(p^2)$, showing how the rank-one lift converts the same address geometry into a larger
retrieval regime. This comparison also exposes what changes when the readout is nonlinear:
applying softmax to the same Welch-scale margin yields the following proposition.

\begin{restatable}[Softmax sharpening at the Welch scale]{proposition}{softmaxsnr}
  \label{prop:softmax-snr}
  Under the setup of \Cref{thm:snr}, let
  $\mu^\star=\sqrt{\eps^\star}$ and take the harmful distractor sign
  $+\mu^\star$. Temperature-$\tau$ softmax satisfies:
  \begin{equation}
    \mathrm{pSNR}_{\mathrm{Soft}}
      =\frac{\exp(2(1-\mu^\star)/\tau)}{T-1}.
    \label{eq:softmax-psnr}
  \end{equation}
  It is therefore ideally retrievable precisely when:
  \begin{equation}
    T-1
    <\exp\!\left(\frac{2(1-\mu^\star)}{\tau}\right).
  \end{equation}
  For $T\gg p$, this gives the approximate capacity:
  \begin{equation}
    T_{\mathrm{Soft}}(p,\tau)
    \approx
    1+\exp\!\left(\frac{2(1-1/\sqrt p)}{\tau}\right).
    \label{eq:softmax-capacity-asymptotic}
  \end{equation}
  Equivalently, any finite target length with $\mu^\star<1$ satisfies the criterion whenever:
  \begin{equation}
    \frac{1}{\tau}
    >
    \frac{\log(T-1)}{2(1-\mu^\star)}.
    \label{eq:softmax-required-scale}
  \end{equation}
\end{restatable}

The softmax partition is common to all coefficients and cancels from pSNR; exponentiation then
sharpens the remaining angular margin. Every fixed positive temperature gives a finite idealized
capacity; for a fixed positive score margin, this capacity is exponential in $1/\tau$. Conversely,
writing $\beta_T=1/\tau$, an unconstrained inverse temperature removes any finite dimension-only
ceiling, with $\beta_T=\Theta(\log T)$ for a fixed positive margin.
This recovers, from a complementary retrieval perspective, the critical logarithmic scale identified
for long-context softmax dynamics by \citet{chen25longcontext}. In our framework the law follows
directly from Welch-scale interference and the pSNR criterion: the matched score must be amplified
logarithmically to remain above the aggregate power of a growing number of distractors. Proofs are
given in \Cref{app:optimalsnr,app:softmaxsnr}.


An alternative perspective on capacity comes from associative-memory theory, which asks how
many patterns remain stable and retrievable under a specified update rule, noise model, and success
criterion~\citep{krotov25mam}. Higher-order and exponential interactions can raise these dynamical
storage thresholds
from polynomial to exponential regimes
~\citep{krotov16dense,demircigil17model,lucibello24exponential}. Our $\eps$-capacity measures the quality of the address packing at a fixed interference tolerance. Under the idealized orthogonal-value readout above, that packing directly determines pSNR and hence whether associative recall succeeds. A nonlinear readout can amplify any positive target--distractor margin, while an exact collision remains irreducible. The specified update rule, noise model, and optimization then determine how closely empirical recall realizes this geometric capacity.

\section{Hardware-Aligned Implementation} \label{sec:hardware}
The capacity story of \Cref{sec:expressivity} is geometric, but its practical reach depends on how
cleanly the feature map maps onto modern accelerators. Chunkwise parallel training is the standard
scaffold for linear-time recurrent attention: split a length-$\seql$ sequence into chunks of size $C$,
yielding
$N_C\defeq\lceil \seql/C\rceil$ chunks with the final chunk padded when necessary, run intra-chunk
attention on tensor cores, and propagate the inter-chunk state by a sequential prefix
sum~\citep{hua22transformerquality}. FLA's chunk-parallel linear attention~\citep{yang24fla}
and Gated DeltaNet's non-commutative Householder transition in the compact WY
representation~\citep{schreiber89compactwy,yang25gateddeltanet} share this decomposition; what differs
is the inter-chunk update. GDN requires a
$C\!\times\!C$ triangular solve per chunk, expanding the per-token state into
high-bandwidth memory (HBM) and tying chunk size to the solve's conditioning. Mamba2's
data-dependent transition is likewise sequential at the chunk level. An additive feature-state recurrence
avoids both costs and admits a parallel scan that lowers inter-chunk depth from $\Theta(N_C)$ to
$\gO(\log N_C)$.

\subsection{From rank-one rays to low-rank PSD packing}
\label{ssec:hardware-psd-maps}

From this point onward, unqualified references to a \emph{KATA mixer} or \emph{KATA kernel}
denote PSD-resident features. Orthant and Lorentz variants are named explicitly. The full
construction uses the rank-one map $\psi(\vu)=\vu\vu^\top$. A direct outer-product layout has $d^2$ entries for head dimension $d$,
or $n=d(d+1)/2$ unique coordinates after symmetry packing. Thus $n_\psi=n$ for the full PSD
map. Linear attention must retain the
feature--value state:
\[
  \mS_t=\Psi(\mK_{1:t})^\top\mV_{1:t}\in\R^{n\times d_v},
\]
so its dominant storage is $n d_v$; the feature width alone understates the memory cost. At $d=64$, the direct and packed
widths are already $4096$ and $2080$, respectively. Multiplication by $d_v$ places the full PSD map
at the edge of hardware-friendly state sizes even at this standard head dimension.

We therefore test two reduced PSD families. Let $g\in\mathbb{N}$ divide $d$, set $m=d/g$, and
partition each projected query or key into $g$ blocks
$\vu=(\vu_1,\ldots,\vu_g)$ with $\vu_a\in\R^m$. KATA-$\Sigma g$ sums the block outer products
into one PSD feature:
\[
  \psi_{\Sigma g}(\vu)
  =\sum_{a=1}^{g}\vu_a\vu_a^\top\in\sS_+^m,
  \qquad
  \lrangle{\psi_{\Sigma g}(\vq),\psi_{\Sigma g}(\vk)}_F
  =\sum_{a,b=1}^{g}(\vq_a^\top\vk_b)^2.
\]
Its dominant state size is:
\[
  \frac{m(m+1)}{2}d_v
  =\frac{(d/g)(d/g+1)}{2}d_v.
\]
Thus KATA-$\Sigma g$ compresses $g$ outer products into a single low-rank PSD factor.

KATA-M$g$ instead retains one PSD factor per block. The raw inner product decomposes as
$\vq^\top\vk=\sum_{a=1}^{g}\vq_a^\top\vk_a$, while the blockwise PSD map is:
\[
  \psi_{\mathrm{M}g}(\vu)
  =\bigl(\vu_1\vu_1^\top,\ldots,\vu_g\vu_g^\top\bigr)
  \in(\sS_+^m)^g,
  \qquad
  \lrangle{\psi_{\mathrm{M}g}(\vq),\psi_{\mathrm{M}g}(\vk)}
  =\sum_{a=1}^{g}(\vq_a^\top\vk_a)^2.
\]
This maps into the Cartesian product of $g$ copies of $\sS_+^{d/g}$ and has dominant state size:
\[
  g\frac{m(m+1)}{2}d_v
  =\frac{d(d/g+1)}{2}d_v.
\]
Both families recover the full rank-one PSD map at $g=1$. A \emph{Delta} prefix adds the delta
rule, a \emph{Gated} prefix adds the learned gate, and \emph{GatedDelta} combines both without
changing the underlying PSD feature geometry. Their recurrences and implementation details are
specified in \Cref{app:variants,app:spd-delta,app:mechanism-archs}.

\paragraph{Chunkwise forward.} Let $n_\psi$ denote the active feature dimension. We partition the sequence into chunks of size $C$ and write the per-chunk feature matrices as
$\Psi(\mQ_{[t]}),\Psi(\mK_{[t]})\in\R^{C\times n_\psi}$ and $\mV_{[t]}\in\R^{C\times d_v}$. The
inter-chunk state update reads:
\begin{equation}
  \mS_{[t+1]}  = \mS_{[t]} + \Psi(\mK_{[t]})^\top \mV_{[t]},\quad
  \mZ_{[t+1]}  = \mZ_{[t]} + \Psi(\mK_{[t]})^\top \vone_C. \label{eq:chunk-SZ}
\end{equation}
With $\mM\in\{0,1\}^{C\times C}$ the lower-triangular causal mask and
$\mP=(\Psi(\mQ_{[t]})\Psi(\mK_{[t]})^\top)\odot\mM$, the chunk output is:
\begin{align}
  \mN_{[t]} & = \Psi(\mQ_{[t]})\mS_{[t]} + \mP\,\mV_{[t]} \\
  \mD_{[t]} & = \Psi(\mQ_{[t]})\mZ_{[t]} + \mP\,\vone_C   \\
  \mO_{[t]} & = \diag(\mD_{[t]})^{-1}\mN_{[t]}
  \label{eq:chunk-out}
\end{align}
Every operation in this pass is a dense general matrix multiplication (GEMM), an element-wise
op, or a warp reduction. Relative to Gated DeltaNet's WY representation, which forms and
inverts a $C\!\times\!C$ unitriangular
matrix per chunk, the inter-chunk cost drops from $\gO(Cd_k^2 + d_k^2 d_v)$ to $\gO(Cn_\psi d_v)$,
and available SRAM sets the chunk-size limit. Triangular-solve conditioning no longer constrains
it. The full pass is summarized in \Cref{alg:kata}.

\subsection{Tree-scan inter-chunk reduction}\label{ssec:treescan}
\begin{figure*}[t]
  \centering
  \begin{tikzpicture}[
      scale=0.78, transform shape,
      font=\sffamily,
      >=Stealth,
      chunk/.style={rectangle, draw=black!60, thick, fill=red!15, text width=2.15cm, align=center, rounded
          corners, minimum height=1.0cm},
          localstate/.style={rectangle, draw=black!60, thick, fill=yellow!30,
          text width=1.0cm, align=center, rounded corners, minimum height=1cm},
          scannode/.style={circle,
          draw=black!60, thick, fill=cyan!20, minimum size=1.0cm, inner sep=0pt},
          prefnode/.style={rectangle, draw=black!60, thick, fill=purple!20, text width=1.5cm, align=center, rounded corners, minimum height=0.8cm},
          outputblock/.style={rectangle, draw=black!60, thick, fill=green!20, text
          width=3.05cm, align=center, rounded corners, minimum height=1cm},
          finalout/.style={rectangle, draw=black!60, thick, fill=gray!15, text width=1.5cm, align=center, rounded corners, minimum
          height=0.8cm},
      arrow/.style={->, thick, draw=black!80}, crossarrow/.style={->, thick, draw=black!80},
      skiparrow/.style={->, thick, dashed, draw=red!60}, shiftarrow/.style={->, thick, draw=purple!80},
      layerlabel/.style={text width=4.5cm, align=left, font=\sffamily\color{black!80}} ]
    \def\xa{0}
    \def\xb{4.2}
    \def\xc{8.4}
    \def\xd{12.6}
    \def\xlegend{16.8}

    \def\yIn{9.75}
    \def\yL{8.25}
    \def\ySone{6.25}
    \def\yStwo{4.75}
    \def\yP{3.5}
    \def\yOutBlock{1.75}
    \def\yFinal{0.0}

    \node[chunk] (IN0) at (\xa, \yIn) {$\mQ_{[0]}, \mK_{[0]}, \mV_{[0]}$};
    \node[chunk] (IN1) at (\xb, \yIn) {$\mQ_{[1]}, \mK_{[1]}, \mV_{[1]}$};
    \node[chunk] (IN2) at (\xc, \yIn) {$\mQ_{[2]}, \mK_{[2]}, \mV_{[2]}$};
    \node[chunk] (IN3) at (\xd, \yIn) {$\mQ_{[3]}, \mK_{[3]}, \mV_{[3]}$};

    \node[layerlabel] at (\xlegend, \yIn) {\textbf{Input Sequence}\\Split into $N_C$ chunks\\Length $C$};

    \node[localstate] (L0) at (\xa, \yL) {$\mS_{[0]}^{\text{loc}}$};
    \node[localstate] (L1) at (\xb, \yL) {$\mS_{[1]}^{\text{loc}}$};
    \node[localstate] (L2) at (\xc, \yL) {$\mS_{[2]}^{\text{loc}}$};
    \node[localstate] (L3) at (\xd, \yL) {$\mS_{[3]}^{\text{loc}}$};

    \foreach \i in {0,1,2,3} {
        \draw[arrow] (IN\i) -- (L\i) node[midway, right, font=\scriptsize] {$\Psi(\mK_{[\i]})^\top\mV_{[\i]}$};
      }

    \node[layerlabel] at (\xlegend, \yL) {\textbf{Intra-Chunk State}\\Computed in parallel\\Time: $\mathcal{O}(C)$};

    \node[scannode] (S10) at (\xa, \ySone) {$\oplus$};
    \node[scannode] (S11) at (\xb, \ySone) {$\oplus$};
    \node[scannode] (S12) at (\xc, \ySone) {$\oplus$};
    \node[scannode] (S13) at (\xd, \ySone) {$\oplus$};

    \draw[arrow] (L0) -- (S10);
    \draw[arrow] (L1) -- (S11);
    \draw[arrow] (L2) -- (S12);
    \draw[arrow] (L3) -- (S13);

    \draw[crossarrow] (L0) -- (S11);
    \draw[crossarrow] (L1) -- (S12);
    \draw[crossarrow] (L2) -- (S13);

    \node[scannode] (S20) at (\xa, \yStwo) {$\oplus$};
    \node[scannode] (S21) at (\xb, \yStwo) {$\oplus$};
    \node[scannode] (S22) at (\xc, \yStwo) {$\oplus$};
    \node[scannode] (S23) at (\xd, \yStwo) {$\oplus$};

    \draw[arrow] (S10) -- (S20);
    \draw[arrow] (S11) -- (S21);
    \draw[arrow] (S12) -- (S22);
    \draw[arrow] (S13) -- (S23);

    \draw[crossarrow] (S10) -- (S22);
    \draw[crossarrow] (S11) -- (S23);

    \begin{scope}[on background layer]
      \node[fill=blue!5, rounded corners, fit=(S10) (S23) (S13) (S20), inner sep=15pt] (scanbox) {};
    \end{scope}

    \node[layerlabel] at (\xlegend, 6.25) {\textbf{Associative Scan}\\Parallel prefix replaces\\sequential recurrence.\\Time: $\mathcal{O}(\log N_C)$};

    \node[prefnode] (P0) at (\xa, \yP) {$\mS_{[0]}^{\text{pre}} = \mathbf{0}$};
    \node[prefnode] (P1) at (\xb, \yP) {$\mS_{[1]}^{\text{pre}}$};
    \node[prefnode] (P2) at (\xc, \yP) {$\mS_{[2]}^{\text{pre}}$};
    \node[prefnode] (P3) at (\xd, \yP) {$\mS_{[3]}^{\text{pre}}$};

    \draw[shiftarrow] (S20) -- (P1) node[midway, above left, font=\scriptsize, text=black] {shift};
    \draw[shiftarrow] (S21) -- (P2) node[midway, above left, font=\scriptsize, text=black] {shift};
    \draw[shiftarrow] (S22) -- (P3) node[midway, above left, font=\scriptsize, text=black] {shift};

    \node[layerlabel] at (\xlegend, \yP) {\textbf{Inter-Chunk Prefix}\\$\mS^{\text{pre}}_{[t]} = \sum_{c<t}\mS^{\text{loc}}_{[c]}$};

      \foreach \i in {0,1,2,3} {
      \node[outputblock] (OUT\i) at (\i*4.2, \yOutBlock) {
      \footnotesize
      \textbf{Inter:} $\Psi(\mQ_{[\i]})\mS_{[\i]}^{\text{pre}}$\\[1mm]
      \textbf{Intra:} $\Psi(\mQ_{[\i]})\Psi(\mK_{[\i]})^\top,\mV_{[\i]}$
    };

    \draw[arrow] (P\i) -- (OUT\i);

    \draw[skiparrow, rounded corners=5pt]
    (IN\i.west) -- ++(-1,0) |- (OUT\i.west);
    }

    \node[layerlabel] at (\xlegend, \yOutBlock) {\textbf{Combine Chunks}\\Time: $\mathcal{O}(C^2)$};

    \node[finalout] (F0) at (\xa, \yFinal) {$\mO_{[0]}$};
    \node[finalout] (F1) at (\xb, \yFinal) {$\mO_{[1]}$};
    \node[finalout] (F2) at (\xc, \yFinal) {$\mO_{[2]}$};
    \node[finalout] (F3) at (\xd, \yFinal) {$\mO_{[3]}$};

    \foreach \i in {0,1,2,3} {
        \draw[arrow] (OUT\i) -- (F\i) node[midway, right, font=\scriptsize] {Add};
      }

    \node[layerlabel] at (\xlegend, \yFinal) {\textbf{Final Embeddings}};

  \end{tikzpicture}
  \caption{\textbf{Chunkwise associative scan.}
  Each chunk forms its local state in parallel. An exclusive scan produces inter-chunk prefixes in
  $\lceil\log_2 N_C\rceil$ rounds, and each output kernel combines its prefix with the local
  $C\!\times\!C$ causal term. Only $\mS$ is shown; $\mZ$ follows the same scan.}
  \label{fig:tree-scan}
\end{figure*}
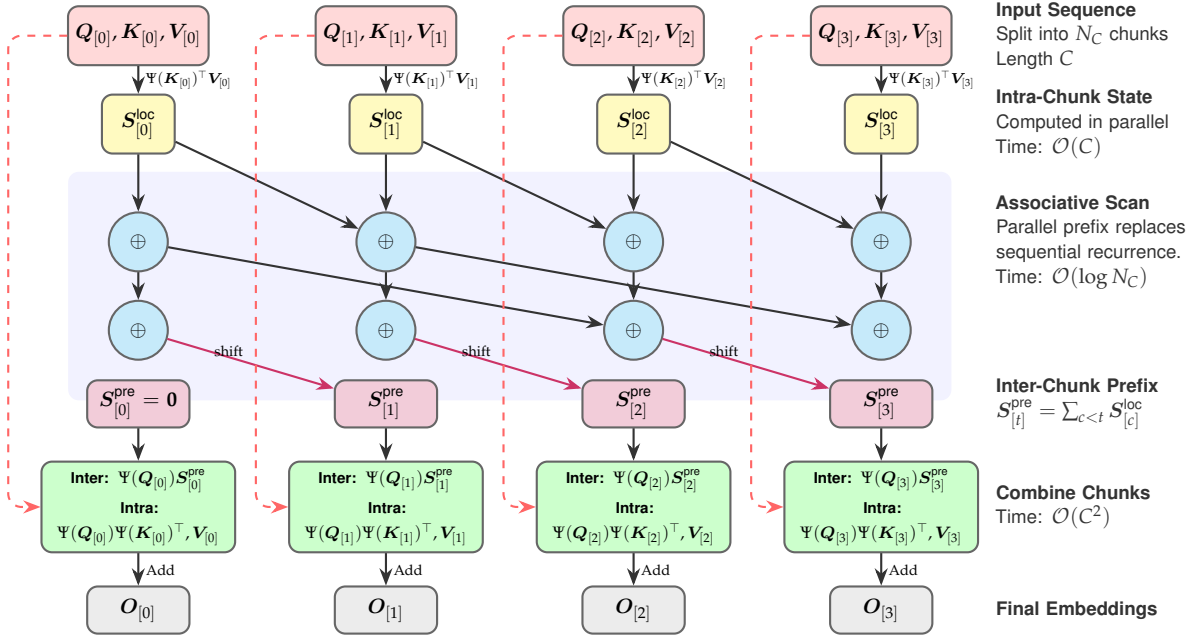

A useful observation is that the recurrence \cref{eq:chunk-SZ} is an associative prefix sum across
the $N_C$ chunks. Existing chunk-parallel implementations, including the FLA and Gated DeltaNet
chunk paths, traverse this prefix sequentially with depth $\Theta(N_C)$. Because
$(\mS_\text{loc},\mZ_\text{loc})$ is a vector-space cumulative sum, however, the same prefix admits
a parallel scan with depth $\gO(\log_2 N_C)$. Our forward decomposes
into three Triton kernels dispatched in series: a per-chunk \emph{state} kernel with
$B\!\cdot\!H\!\cdot\!N_C$ programs in parallel, a \emph{cumulative} pass that performs a
associative scan in $\lceil\log_2 N_C\rceil$ rounds with $N_C$-wide parallelism per
round, and a per-chunk \emph{output} kernel that combines the inter-chunk prefix with the local
$C\!\times\!C$ causal pattern. The resulting wall-clock model is:
\begin{align*}
  t_\text{total}
  &=\underbrace{\gO\lrbp{C n_\psi d_v}}_{\text{state}}
  +\underbrace{\gO\lrbp{\log_2 N_C\cdot d_v}}_{\text{tree-scan}}
  +\underbrace{\gO\lrbp{C n_\psi d_v+C^2(n_\psi+d_v)}}_{\text{output}},
\end{align*}
replacing the linear-scan baseline's $\gO(N_C d_v)$ middle term. The $n_\psi$ feature slices are independent and execute in parallel.

At $\seql{=}2048$ and $C{=}32$, the scan depth shrinks from $N_C{=}64$ to
$\log_2 N_C{=}6$. Its arithmetic cost is negligible relative to the HBM write of
$\mS_\text{loc}$ at chunk boundaries, which is the actual bottleneck.\par

\paragraph{Occupancy at small batch.} The advantage is largest precisely where a sequential scan
starves the GPU. At $B{=}1$, the $B\!\cdot\!H{=}16$ programs of a sequential chunked scan, as in
FLA or Gated DeltaNet, cannot fill the H100's $132$ SMs, so the forward is occupancy-bound and
plateaus near $25$\,Mtok/s. The tree scan instead exposes $B\!\cdot\!H\!\cdot\!N_C$ programs, which
scale with the actual work. \Cref{fig:btsweep_toks,tab:treescan} report the full
$B\times\seql$ grid and its $B=1$ slice, respectively. For the linear feature
($\psi(\vx)=\vx$, with a matched $d\times d$ state), the associative scan delivers
$\mathbf{1.9}$--$\mathbf{2.9\times}$ the throughput of the sequential linear-chunk kernel from $2048$
to $131$k tokens, and up to $7.5\times$ the throughput of Gated DeltaNet at $\seql=2048$.\par

\begin{wrapfigure}{r}{0.36\textwidth}
    \centering
    \includegraphics[width=\linewidth]{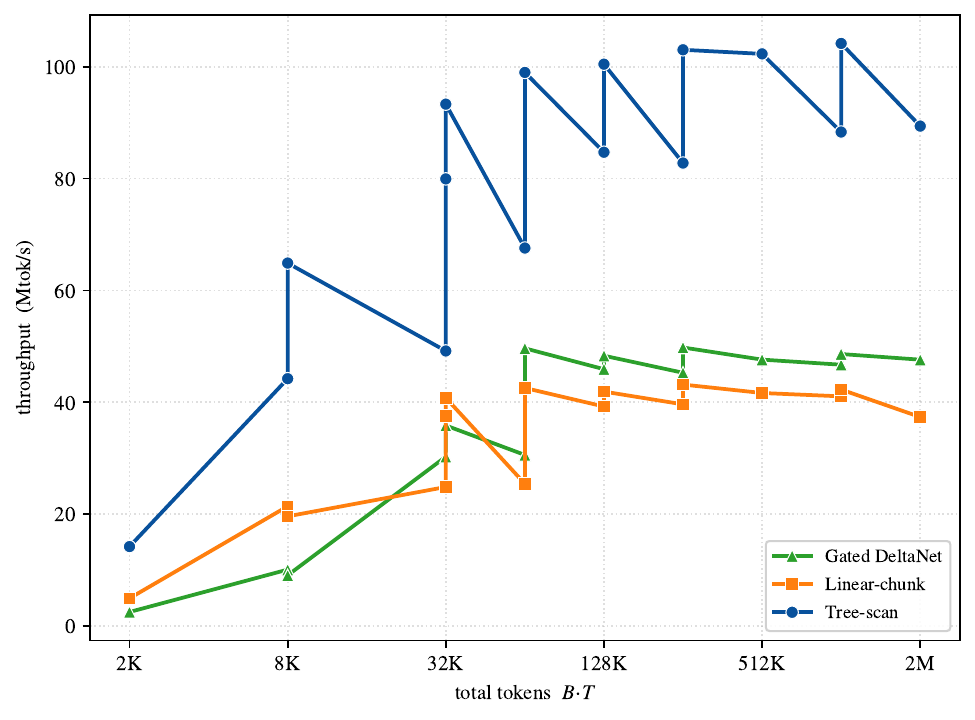}
    \caption{\textbf{Forward throughput over $B\times\seql$.}
    Associative scan, sequential scan, and Gated DeltaNet on H100.}
    \label{fig:btsweep_toks}
\end{wrapfigure}

\paragraph{Why the PSD feature keeps to $\gO(\seql)$ or $\gO(\seql^2)$.} The tree scan is viable only
for the linear feature. A PSD block of width $m=d/g$ has a direct outer-product feature dimension
$m^2$, so a chunk's local state $\psi(\mK_{[c]})^\top\mV_{[c]}$ is an
$m^2\!\times\!d_v$ tile. It occupies $128$\,KB at $g{=}2$ ($m^2{=}1024$) and already exceeds the
H100's $228$\,KB shared-memory budget once the query feature and matrix-multiplication accumulators
are resident. At $g{=}1$ ($m^2{=}4096$), it occupies $1.5$\,MB. A scan must moreover materialize
every chunk's state to HBM, requiring $\gO(N_C m^2 d_v)$ storage traffic. We therefore do not
tree-scan the PSD recurrence. KATA-M$g$ runs either as a quadratic $\gO(\seql^2)$ kernel, which never
forms the $m^2$ state, or as a linear-state $\gO(\seql)$ kernel, whose
$m^2\!\times\!d_v$ state lives in HBM and is streamed in feature blocks. Realizing the
parallel-prefix speedup for PSD features would require this feature blocking within the scan and is
left to future work. Full timings and SRAM budgets are reported in
\Cref{app:tree-scan,app:hardware-bench}.

\subsection{Synthetic induction and hardware microbenchmarks}
\label{ssec:results-aux}

KATA-$\Sigma2$ solves a $V{=}256$ induction-head probe at every distance up to $2048$ tokens
within an $8$k-step budget, matching or nearly matching softmax with FlashAttention-2 at every scale tested while
avoiding the $\gO(\seql^2)$ memory of dense attention (\Cref{app:induction}). The fused rank-two
KATA-$\Sigma2$ forward scales linearly with $\seql$ at near-peak bandwidth without writing the
per-token $\psi$ tensor to HBM; full timings appear in \Cref{app:hardware-bench}.

\paragraph{Kernel latency and throughput ratios.} \Cref{fig:tsweep} sweeps sequence length on an NVIDIA~H100 for the two
computational forms of the KATA-M$g$ kernel. Across the tested lengths, the quadratic KATA-M1
forward delivers $1.3$--$1.6\times$ the throughput of FlashAttention-2, while KATA-M2 delivers
$1.2$--$1.4\times$ the throughput. Both retain the same $\gO(\seql^2)$ scaling. The exact linear-state PSD
kernel instead trades an $m^2{\times}d_v$ recurrent state, with $m=d/g$, for $\gO(\seql)$ compute.
KATA-M2 reaches parity with its quadratic counterpart near $16$K and is faster by $32$K; KATA-M1
is faster by $128$K. At $131$K, linear-state KATA-M2 reaches $11\times$ FlashAttention-2's forward
throughput while implementing the exact PSD recurrence. Gated DeltaNet remains faster, but its
raw-key feature loses high-entropy needle information (\Cref{tab:niah}). The measured latency scaling
approaches $2\times$ per doubling for the linear kernels and $4\times$ for the quadratic kernels.

\begin{figure}[t]
    \centering
    \includegraphics[width=\textwidth]{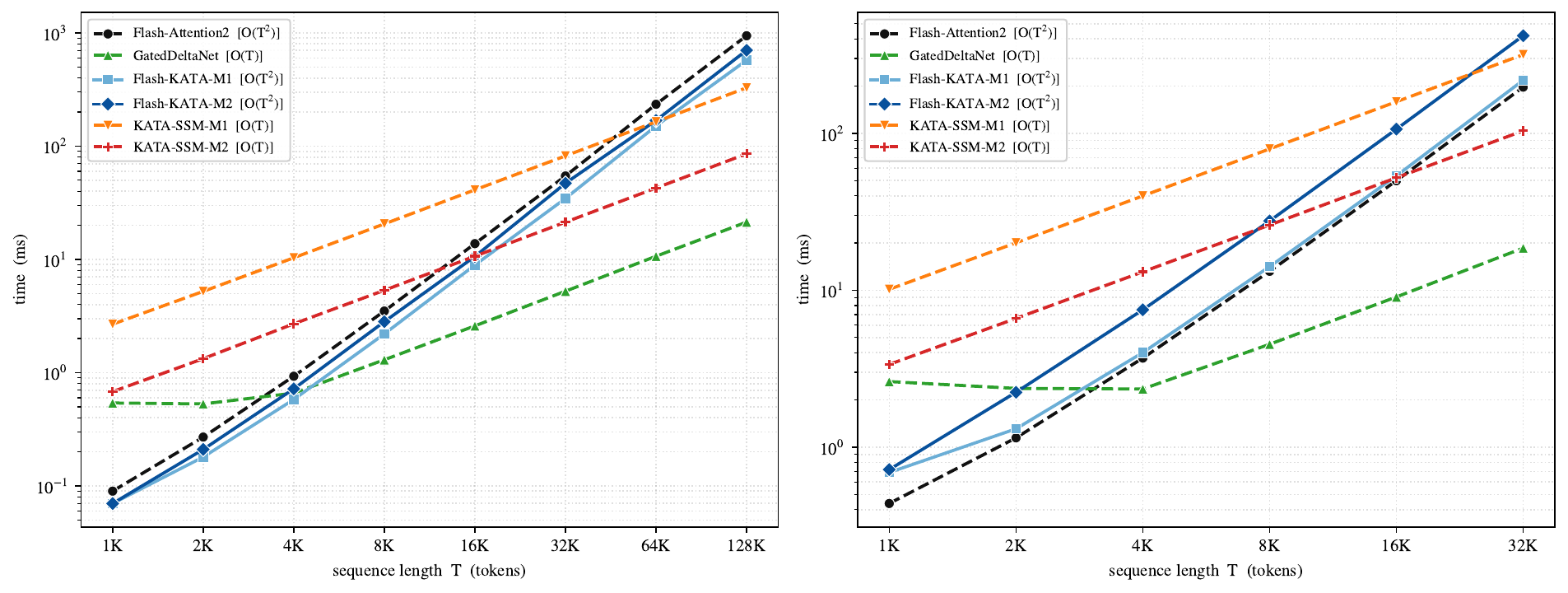}
    \caption{\textbf{KATA-M$g$ latency versus sequence length on an NVIDIA H100.}
    Forward latency (left) and forward-plus-backward latency (right) at
    $B{=}8$, $H{=}16$, $d_\text{head}{=}64$, and bf16 on log--log axes.
    Solid curves are the $\gO(\seql^2)$ kernels, FlashAttention-2 and quadratic KATA-M$g$;
    dashed curves are the $\gO(\seql)$ recurrences, Gated DeltaNet and linear-state KATA-M$g$.
    Lower is better; exact timings appear in
    \Cref{tab:tsweep-fwd,tab:tsweep-bwd}.}
    \label{fig:tsweep}
\end{figure}

\section{Experiments} \label{sec:experiments}
The theory of \Cref{sec:expressivity} predicts that associative-recall behavior separates by cone,
and the kernel of \Cref{sec:hardware} makes the corresponding feature maps cheap to run. We test
both predictions on three synthetic tasks that isolate distinct failure modes: multi-query
associative recall (MQAR) with controlled length extrapolation, a repeated-key overwrite probe,
and a long-range induction-head probe. We then scale the recall test to pretrained 340M-parameter language models
with a needle-in-a-haystack suite. Only KATA variants instantiated on the PSD cone as described in \Cref{ssec:hardware-psd-maps} are compared against vanilla linear
attention~\citep{katharopoulos20transformers}, Performer~\citep{choromanski21performers},
DeltaNet~\citep{yang24deltanet}, Gated DeltaNet~\citep{yang25gateddeltanet},
Mamba2~\citep{dao24mamba2}, and softmax attention. Concrete KATA configurations are named by feature family, KATA-$\Sigma g$ or KATA-M$g$.
The prefixes \emph{Delta}, \emph{Gated}, and \emph{GatedDelta} add the content erase,
multiplicative state gate, or both; for example, GatedDeltaKATA-M2 places PSD features inside the Gated DeltaNet raw-key
recurrence. Per-architecture details and reproduction recipes
are in \Cref{app:mechanism-archs,app:experiments-extended}. Experimental configurations
and the early-stopping criterion are matched across methods.

\subsection{Multi-query associative recall (MQAR)}
\label{ssec:results-capacity}

We use the canonical Zoology MQAR setup~\citep{arora23zoology}. Training mixes streams up to
$K_\text{train}{=}64$ at $\seql=256$; evaluation grids the OOD slice $K\in\{128,256,512,768,1024\}$
over $\seql\in\{512,\dots,4096\}$, so the largest evaluation slice is $\mathbf{16\times}$ the longest
training sequence and the largest training $K$. We report the best accuracy across the three-point
sweep $\mathrm{lr}\in\{10^{-3},3\!\times\!10^{-3},10^{-2}\}$. KATA-Orthant and KATA-Lorentz drop
into the same backbone as every other row, varying only the feature map.
\begin{table*}[t]
  \centering
  \fontsize{7.7pt}{9.2pt}\selectfont
  \setlength{\tabcolsep}{2pt}
  \begin{tabular}{l|c|rr|cccccc|ccccc}
    \toprule
    \multirow{3}{*}{Method}
      & \multirow{3}{*}{$d_h$}
      & \multirow{3}{*}{\#params}
      & \multirow{3}{*}{state (\#dim)}
      & \multicolumn{6}{c|}{MQAR length extrapolation}
      & \multicolumn{5}{c}{Repeated-key overwrite} \\
    \cmidrule(lr){5-10}\cmidrule(lr){11-15}
      & & & &
      \multicolumn{5}{c|}{\emph{OOD}}
      & \multirow{2}{*}{Avg.}
      & \multicolumn{1}{c|}{\emph{ID}}
      & \multicolumn{3}{c|}{\emph{OOD}}
      & \multirow{2}{*}{Avg.} \\
    \cmidrule(lr){5-9}\cmidrule(lr){11-11}\cmidrule(lr){12-14}
      & & & &
      $\seql{=}512$ & $\seql{=}1024$ & $\seql{=}2048$ & $\seql{=}3072$ & $\seql{=}4096$ &
      $K{=}64$ & $K{=}128$ & $K{=}256$ & $K{=}384$ & \\
    \midrule
    Softmax attn.
      & 64 & 1{,}132{,}928 & 1{,}048{,}576$^\star$
      & 1.000 & 1.000 & 1.000 & 1.000 & 1.000 & 1.000
      & 0.513 & 0.506 & 0.499 & 0.491 & 0.502 \\
    Softmax+RoPE ($\theta{=}10^4$)
      & 64 & 1{,}132{,}928 & 1{,}048{,}576$^\star$
      & 0.994 & 0.658 & 0.134 & 0.036 & 0.012 & 0.367
      & 0.998 & 0.890 & 0.403 & 0.160 & 0.613 \\
    Softmax+RoPE ($\theta{=}500$)
      & 64 & 1{,}132{,}928 & 1{,}048{,}576$^\star$
      & 0.741 & 0.046 & 0.000 & 0.001 & 0.001 & 0.158
      & 0.991 & 0.670 & 0.140 & 0.025 & 0.457 \\
    \midrule
    RWKV7
      & 64 & 1{,}162{,}880 & 8{,}192
      & 0.709 & 0.224 & 0.051 & 0.022 & 0.013 & 0.204
      & 0.820 & 0.341 & 0.084 & 0.030 & 0.319 \\
    Mamba2
      & 64 & 1{,}217{,}548 & 16{,}384
      & 0.845 & 0.413 & 0.173 & 0.094 & \underline{0.054} & 0.316
      & 0.930 & 0.556 & 0.199 & 0.127 & 0.453 \\
    GLA
      & 64 & 1{,}154{,}624 & 8{,}192
      & 0.900 & 0.496 & 0.151 & 0.067 & 0.038 & 0.330
      & 0.768 & 0.317 & 0.082 & 0.030 & 0.299 \\
    DeltaNet
      & 64 & 1{,}134{,}272 & 8{,}192
      & \underline{0.992} & \underline{0.733} & \underline{0.249} & \underline{0.101} & 0.047 & \underline{0.424}
      & \underline{0.984} & \underline{0.691} & \underline{0.306} & \underline{0.171} & \underline{0.538} \\
    Gated DeltaNet
      & 64 & 1{,}167{,}878 & 8{,}192
      & \textcolor{blue}{\textbf{1.000}} & \textcolor{blue}{\textbf{0.944}} & \textcolor{blue}{\textbf{0.515}} & \textcolor{blue}{\textbf{0.279}} & \textcolor{blue}{\textbf{0.170}} & \textcolor{blue}{\textbf{0.582}}
      & \textcolor{blue}{\textbf{0.998}} & \textcolor{blue}{\textbf{0.853}} & \textcolor{blue}{\textbf{0.384}} & \textcolor{blue}{\textbf{0.217}} & \textcolor{blue}{\textbf{0.613}} \\
    \cmidrule(lr){1-13}
    RWKV7
      & 128 & 2{,}489{,}600 & 32{,}768
      & 0.982 & 0.670 & 0.192 & 0.070 & 0.032 & 0.389
      & 0.881 & 0.376 & 0.089 & 0.031 & 0.344 \\
    Mamba2
      & 128 & 2{,}629{,}132 & 32{,}768
      & 0.824 & 0.336 & 0.066 & 0.021 & 0.008 & 0.251
      & \underline{0.936} & 0.595 & 0.224 & 0.144 & 0.475 \\
    GLA
      & 128 & 2{,}505{,}856 & 32{,}768
      & \underline{0.993} & 0.824 & 0.369 & 0.178 & 0.100 & 0.493
      & 0.897 & 0.404 & 0.105 & 0.037 & 0.361 \\
    DeltaNet
      & 128 & 2{,}432{,}384 & 32{,}768
      & \textcolor{blue}{\textbf{1.000}} & \underline{0.990} & \underline{0.526} & \underline{0.244} & \underline{0.138} & \underline{0.580}
      & \textcolor{blue}{\textbf{1.000}} & \underline{0.971} & \underline{0.585} & \underline{0.347} & \underline{0.726} \\
    Gated DeltaNet
      & 128 & 2{,}565{,}126 & 32{,}768
      & \textcolor{blue}{\textbf{1.000}} & \textcolor{blue}{\textbf{0.997}} & \textcolor{blue}{\textbf{0.670}} & \textcolor{blue}{\textbf{0.352}} & \textcolor{blue}{\textbf{0.186}} & \textcolor{blue}{\textbf{0.641}}
      & \textcolor{blue}{\textbf{1.000}} & \textcolor{blue}{\textbf{0.978}} & \textcolor{blue}{\textbf{0.595}} & \textcolor{blue}{\textbf{0.358}} & \textcolor{blue}{\textbf{0.733}} \\
    \cmidrule(lr){1-13}
    RWKV7
      & 256 & 2{,}489{,}600 & 65{,}536
      & \textcolor{blue}{\textbf{1.000}} & 0.986 & 0.661 & 0.332 & 0.175 & 0.631
      & 0.938 & 0.434 & 0.110 & 0.038 & 0.380 \\
    GLA
      & 256 & 5{,}798{,}144 & 65{,}536
      & \textcolor{blue}{\textbf{1.000}} & \underline{0.996} & \underline{0.902} & \textcolor{blue}{\textbf{0.689}} & \textcolor{blue}{\textbf{0.491}} & \textcolor{blue}{\textbf{0.816}}
      & 0.927 & 0.450 & 0.132 & 0.050 & 0.390 \\
    DeltaNet
      & 256 & 5{,}520{,}128 & 65{,}536
      & \textcolor{blue}{\textbf{1.000}} & \textcolor{blue}{\textbf{1.000}} & 0.865 & 0.535 & 0.342 & 0.748
      & \textcolor{blue}{\textbf{1.000}} & \textcolor{blue}{\textbf{0.998}} & \textcolor{blue}{\textbf{0.851}} & \textcolor{blue}{\textbf{0.569}} & \textcolor{blue}{\textbf{0.855}} \\
    Gated DeltaNet
      & 256 & 6{,}047{,}750 & 65{,}536
      & \textcolor{blue}{\textbf{1.000}} & \textcolor{blue}{\textbf{1.000}} & \textcolor{blue}{\textbf{0.909}} & \underline{0.618} & \underline{0.419} & \underline{0.789}
      & \underline{0.999} & \underline{0.992} & \underline{0.726} & \underline{0.419} & \underline{0.784} \\
    \midrule
    KATA-$\Sigma4$
      & 64 & 1{,}134{,}016 & 17{,}408
      & \textcolor{blue}{\textbf{1.000}} & \underline{0.997} & 0.865 & 0.610 & 0.407 & 0.776
      & 0.517 & 0.493 & 0.364 & 0.226 & 0.400 \\
    KATA-$\Sigma2$
      & 64 & 1{,}134{,}016 & 67{,}584
      & \textcolor{blue}{\textbf{1.000}} & \textcolor{blue}{\textbf{1.000}} & \textcolor{blue}{\textbf{1.000}} & 0.995 & 0.958 & 0.991
      & 0.518 & 0.505 & 0.444 & 0.342 & 0.452 \\
    KATA-M4
      & 64 & 1{,}134{,}016 & 69{,}632
      & \textcolor{blue}{\textbf{1.000}} & \textcolor{blue}{\textbf{1.000}} & \underline{0.982} & 0.877 & 0.707 & 0.913
      & 0.515 & 0.484 & 0.346 & 0.219 & 0.391 \\
    KATA-M2
      & 64 & 1{,}134{,}016 & 135{,}168
      & \textcolor{blue}{\textbf{1.000}} & \textcolor{blue}{\textbf{1.000}} & \textcolor{blue}{\textbf{1.000}} & 0.997 & 0.970 & 0.993
      & 0.521 & 0.509 & 0.492 & 0.436 & 0.489 \\
    KATA-M2+RoPE
      & 64 & 1{,}134{,}016 & 135{,}168
      & \underline{0.979} & 0.384 & 0.040 & 0.007 & 0.003 & 0.283
      & \underline{0.992} & 0.696 & 0.087 & 0.012 & 0.447 \\
    DeltaKATA-M2
      & 64 & 1{,}134{,}272 & 135{,}168
      & \textcolor{blue}{\textbf{1.000}} & \textcolor{blue}{\textbf{1.000}} & \textcolor{blue}{\textbf{1.000}} & 0.995 & 0.955 & 0.990
      & \textcolor{blue}{\textbf{1.000}} & \underline{0.999} & \underline{0.992} & \underline{0.922} & \underline{0.978} \\
    GatedKATA-M2
      & 64 & 1{,}134{,}276 & 135{,}168
      & \textcolor{blue}{\textbf{1.000}} & \textcolor{blue}{\textbf{1.000}} & \textcolor{blue}{\textbf{1.000}} & 0.996 & 0.965 & 0.992
      & 0.991 & 0.662 & 0.193 & 0.074 & 0.480 \\
    GatedDeltaKATA-M2
      & 64 & 1{,}134{,}532 & 135{,}168
      & \underline{0.979} & 0.699 & 0.258 & 0.119 & 0.062 & 0.423
      & \textcolor{blue}{\textbf{1.000}} & \underline{0.999} & 0.983 & 0.859 & 0.960 \\
    KATA-M1
      & 64 & 1{,}134{,}016 & 266{,}240
      & \textcolor{blue}{\textbf{1.000}} & \textcolor{blue}{\textbf{1.000}} & \textcolor{blue}{\textbf{1.000}} & \underline{0.998} & \underline{0.985} & \underline{0.997}
      & 0.521 & 0.511 & 0.503 & 0.483 & 0.505 \\
    DeltaKATA-M1
      & 64 & 1{,}134{,}272 & 266{,}240
      & \textcolor{blue}{\textbf{1.000}} & \textcolor{blue}{\textbf{1.000}} & \textcolor{blue}{\textbf{1.000}} & \textcolor{blue}{\textbf{1.000}} & \textcolor{blue}{\textbf{0.999}} & \textcolor{blue}{\textbf{1.000}}
      & \textcolor{blue}{\textbf{1.000}} & \textcolor{blue}{\textbf{1.000}} & \textcolor{blue}{\textbf{0.999}} & \textcolor{blue}{\textbf{0.998}} & \textcolor{blue}{\textbf{0.999}} \\
    \bottomrule
  \end{tabular}
  \caption{\textbf{MQAR length extrapolation and repeated-key overwrite.}
  Accuracy (higher is better); Avg. is the unweighted mean over the displayed settings.
  Model size and state dimension use the MQAR configuration.
  $^\star$Softmax reports its KV-cache size at $\seql{=}4096$.
  Within each task and block, the best and second-best fixed-state results are shown in bold blue and underlined, respectively.}
  \label{tab:mqar-overwrite}
  \label{tab:overwrite}
  \label{tab:mqar}
\end{table*}

In-distribution accuracy saturates for every fixed-state row, so the differences arise entirely on
the out-of-distribution slices. The orthant and Lorentz baselines collapse just past
$K\!\approx\!d_\text{head}$, where the linear Welch and Rankin ceilings of
\Cref{thm:cap-lor,thm:welchbound} predict. The PSD variants extend recall into the
$16\times$-extrapolation regime. At the largest slice, KATA-M1 retains $0.985$ accuracy with a
$266{,}240$-entry fixed state, roughly one quarter of softmax's $1{,}048{,}576$ KV-cache entries;
DeltaKATA-M1 reaches $0.999$ at the same state size. GatedDeltaKATA-M2 exposes the recurrence
trade-off: its average overwrite accuracy rises to $0.960$, while its MQAR average falls to $0.423$.
The convex-gate ablation, which caps $\norm{\mS_t}_F$ inside the convex hull of instantaneous
bindings, halves the OOD accuracy of the additive form (full table in \Cref{app:experiments-extended}).

\subsection{Repeated-key overwrite}
\label{ssec:results-overwrite}

Each of $K$ unique keys is written twice with two \emph{different} values, and the model must
retrieve the most recent binding. Training mixes $K\in\{4,8,16,32\}$ at sequence lengths up to
$256$; evaluation extends to $K\in\{64,128,256,384\}$ with $\seql\le 3072$. The backbone is the
two-layer Zoology transformer with a width-$3$ convolutional mixer, parameter-matched across rows.

For a repeated key, an additive state writes both values along the same feature direction and
therefore contains a term proportional to
$\psi(\vk)(\vv_{\mathrm{old}}+\vv_{\mathrm{new}})^\top$. The readout at that address sees
their superposition without temporal information. RoPE resolves the ambiguity in-distribution by
assigning the two writes different position-dependent features. Because this benchmark always
targets the latest write, the model learns to favor the most recent association. In a general
retrieval problem, the earlier association may remain correct, and this learned recency preference
can select the newer value. The positional shortcut also extrapolates poorly in
\Cref{tab:mqar-overwrite}, where both softmax+RoPE and KATA-M2+RoPE deteriorate rapidly beyond the
training distribution.

The convex output gate of \Cref{prop:output-gate} acts at readout and leaves the relative
contributions stored in $\mS_t$ unchanged. GatedKATA places its learned scalar inside the
recurrence. Unrolling
$\mS_t=\gamma_t\mS_{t-1}+\psi(\vk_t)\vv_t^\top$ gives:
\[
  \mS_t
  =\sum_{s=1}^{t}
    \left(\prod_{r=s+1}^{t}\gamma_r\right)
    \psi(\vk_s)\vv_s^\top.
\]
The recurrent gate assigns every write a multiplicative, token-dependent positional weight. At
$K{=}64$, this weighting raises overwrite accuracy from $0.521$ for KATA-M2 to $0.991$ for
GatedKATA-M2. The result supports a positional interpretation of the scalar gate, while state
magnitude control plays a secondary role. Since the product usually shrinks with age, the gate
favors recent writes.

The delta rule corrects the value stored at the current content address. For a unit-norm feature,
let $\widehat{\vv}_t=\mS_{t-1}^\top\psi(\vk_t)$ be the value currently retrieved at $\vk_t$.
The update is:
\[
  \mS_t
  =\mS_{t-1}
  +\beta_t\psi(\vk_t)\bigl(\vv_t-\widehat{\vv}_t\bigr)^\top.
\]
When $\beta_t=1$, reading the same address after the update returns $\vv_t$ exactly. When the
earlier association should be preserved, a learned $\beta_t$ can suppress the correction. The
update strength therefore controls content-dependent replacement at each write.
DeltaKATA-M1 is nearly perfect across the overwrite sweep, while GatedDeltaKATA-M2 reaches
$0.960$ average overwrite accuracy. Full recurrence details are given in \Cref{app:spd-delta}.

\subsection{Pretraining 340M-parameter language models}
\label{ssec:llm-pretraining}

To test whether the synthetic results persist under realistic pretraining, we train a family of
340M-parameter models
($24$ layers, $d_\text{model}{=}1024$, $16$ heads) for $15$B tokens of
SlimPajama~\citep{soboleva23slimpajama} with the TinyLlama tokenizer~\citep{zhang24tinyllama} at a
$2048$-token context, holding everything fixed and swapping only the token-mixing layer: softmax
attention, Gated DeltaNet, KATA-M1 and KATA-M2 (RoPE positions, RMSNorm on $\vq,\vk$),
DeltaKATA-M1 (KATA-M1 with a DeltaNet content erase), and $\ell_2$-norm and short-conv
ablations. The models remain broadly comparable on zero-shot accuracy, while perplexity reveals a meaningful
fluency spread (\Cref{tab:pretrain}). KATA-M1 trails softmax attention, Gated DeltaNet, and
DeltaKATA-M1 on both perplexity datasets; KATA-M2 is higher still. The recall evaluations therefore
separate direct association from contextual retrieval: NIAH primarily probes stored bindings,
while FDA and SQuAD also require context interpretation and candidate selection.
\begin{table*}[ht]
    \centering\footnotesize\setlength{\tabcolsep}{4pt}
  \begin{tabular}{ll|cc|ccccccc}
    \toprule
    \multicolumn{2}{l|}{\multirow{2}{*}{Method}} & \multicolumn{2}{c|}{Perplexity $\downarrow$} & \multicolumn{7}{c}{Accuracy $\uparrow$} \\
    \cmidrule(lr){3-4}\cmidrule(lr){5-11}
    \multicolumn{2}{l|}{}          & LMBD           & WikiT          & LMBD           & ARC-e          & ARC-c & HeSw           & PIQA           & Wino  & BoolQ \\
    \midrule
    \multicolumn{2}{l|}{Softmax attn.}          & 38.6           & 27.6           & 0.334          & 0.412          & 0.234 & 0.346          & 0.645          & 0.503 & 0.584 \\
    \multicolumn{2}{l|}{Gated DeltaNet}         & \textbf{40.7}  & 28.4           & 0.316          & 0.393          & 0.223 & \textbf{0.349} & 0.639          & 0.519 & 0.562 \\
    \midrule
    \multicolumn{2}{l|}{KATA-M1}     & 43.5           & 29.1           & \textbf{0.319} & 0.399          & 0.222 & 0.340          & \textbf{0.642} & 0.519 & 0.585 \\
    \multicolumn{2}{l|}{DeltaKATA-M1}    & 42.5           & \textbf{28.1}  & 0.309          & \textbf{0.403} & 0.208 & 0.343          & 0.641          & 0.502 & 0.597 \\
    \multicolumn{2}{l|}{KATA-M2}     & 51.8           & 30.0           & 0.281          & 0.385          & 0.230 & 0.331          & 0.635          & 0.512 & 0.614 \\
    \addlinespace[1pt]
    \multicolumn{2}{l|}{\emph{ablations:}}       &                &                &                &                &       &                &                &       &       \\
    \quad & $\ell_2$-norm (KATA-M2)             & 57.1           & 32.2           & 0.277          & 0.389          & 0.230 & 0.323          & 0.627          & 0.499 & 0.598 \\
    \quad & short-conv, no RoPE (KATA-M2)       & 84.6           & 34.0           & 0.233          & 0.363          & 0.230 & 0.313          & 0.625          & 0.509 & 0.615 \\
    \bottomrule
  \end{tabular}
    \caption{\textbf{Pretraining quality of the 340M models.}
    For word-level perplexity, lower is better; for zero-shot accuracy, higher is better
    (\texttt{acc\_norm} for ARC and HellaSwag, \texttt{acc} otherwise; LMBD is LAMBADA).
    Bold marks the best sub-quadratic result in the columns discussed in the text.}
    \label{tab:pretrain}
\end{table*}

\paragraph{Why the delta rule remains useful with RoPE.} RoPE assigns repeated occurrences of the
same key different phases, so it can in principle resolve overwrite within the $2048$-token training
window. The resulting separation depends on the attention kernel. Consider an aligned unit query
and key in one rotary plane with angular rate $\omega$ and relative offset
$\Delta=|j-i|$. Relative to $\Delta=0$, the KATA and softmax factors are:
\[
  r_{\mathrm{KATA}}(\Delta)=\cos^2(\omega\Delta),
  \qquad
  r_{\mathrm{softmax}}(\Delta)
  =\exp\!\left(\frac{\cos(\omega\Delta)-1}{\tau}\right).
\]
For $d_\text{head}=64$, RMS-normalized query and key vectors have norm approximately
$\sqrt{d_\text{head}}$ before learned per-channel gains. Combined with the standard
$1/\sqrt{d_\text{head}}$ logit scaling, this gives an effective
$\tau\approx1/\sqrt{d_\text{head}}=1/8$ on normalized directions
(\Cref{app:rms-temperature}). Taking the nominal slow rotary rate
$\omega=1/\theta_{\mathrm{RoPE}}$, with $\theta_{\mathrm{RoPE}}=10^4$ and
$\Delta=1024$, gives:
\[
  \omega\Delta=0.1024,
  \qquad
  r_{\mathrm{KATA}}=0.9896,
  \qquad
  r_{\mathrm{softmax}}=0.9590.
\]
The KATA score loses about $1.0\%$ across this offset, while the exponential softmax weight loses
about $4.1\%$. Squaring also removes the sign of the rotated inner product, further weakening phase
discrimination. The full RoPE map averages across a spectrum of rates, and this single-plane
calculation isolates the effect of the kernel nonlinearity. The delta correction directly updates
the content address and can therefore improve recall even inside the training window.

The multiplicative state weights of GatedKATA provide another source of learned positional decay
and should strengthen this separation. The 340M-model sweep leaves GatedKATA pretraining and its
interaction with RoPE for future work.

\subsection{Needle-in-a-haystack recall}
\label{ssec:results-niah}

Using the same pretrained models, we evaluate the three RULER single-needle tasks
S-NIAH-1/2/3~\citep{hsieh24ruler}, which hide a word, a $7$-digit number, and a $128$-bit UUID at a
random depth in a distractor haystack and query it with a cloze prompt. The tasks form an
increasing-entropy ladder: a UUID is a high-entropy exact string, the regime that most stresses a
fixed-size recurrent state. We score the full $500$ samples per length with greedy decoding and
normalized substring matching at 1K and 2K, both within the $2048$-token training horizon. We
restrict the main capacity comparison to these in-window lengths because the RoPE-based softmax
and KATA models collapse once evaluation moves beyond the training horizon. This protocol isolates
retrieval capacity from positional length extrapolation; OOD lengths are reported in
\Cref{app:experiments-extended}.
\begin{table*}[t]
  \centering
  \fontsize{7.7pt}{9.2pt}\selectfont
  \setlength{\tabcolsep}{2.2pt}
  \begin{tabular}{l|cc|cc|cc|ccccc}
    \toprule
    \multirow{3}{*}{Method}
      & \multicolumn{6}{c|}{RULER needle-in-a-haystack}
      & \multicolumn{5}{c}{Arora'24 (Based) cloze} \\
    \cmidrule(lr){2-7}\cmidrule(lr){8-12}
      & \multicolumn{2}{c|}{S-NIAH-1 (word)}
      & \multicolumn{2}{c|}{S-NIAH-2 (number)}
      & \multicolumn{2}{c|}{S-NIAH-3 (UUID)}
      & \multirow{2}{*}{SWDE}
      & \multirow{2}{*}{FDA}
      & \multirow{2}{*}{SQuAD}
      & \multirow{2}{*}{TriviaQA}
      & \multirow{2}{*}{NQ} \\
    \cmidrule(lr){2-3}\cmidrule(lr){4-5}\cmidrule(lr){6-7}
      & 1K & 2K & 1K & 2K & 1K & 2K & & & & & \\
    \midrule
    Softmax attn.
      & 1.000 & 0.992 & 1.000 & 1.000 & 0.956 & 0.930
      & 0.450 & 0.170 & 0.160
      & 0.290 & 0.227 \\
    \midrule
    Gated DeltaNet
      & 0.996 & 0.954 & 0.384 & 0.748 & 0.004 & 0.004
      & 0.225 & 0.060 & \underline{0.293} & 0.254 & 0.156 \\
    \midrule
    KATA-M1
      & 0.994 & \textcolor{blue}{\textbf{0.990}} & \underline{0.986} & 0.910 & 0.758 & 0.458
      & \textcolor{blue}{\textbf{0.375}} & 0.061 & 0.041 & \underline{0.258} & \textcolor{blue}{\textbf{0.182}} \\
    DeltaKATA-M1
      & \underline{0.998} & 0.950 & \textcolor{blue}{\textbf{1.000}}
      & \textcolor{blue}{\textbf{0.990}} & \underline{0.908} & \underline{0.584}
      & \underline{0.370} & \textcolor{blue}{\textbf{0.148}} & \textcolor{blue}{\textbf{0.353}}
      & \textcolor{blue}{\textbf{0.268}} & \underline{0.179} \\
    KATA-M2
      & \textcolor{blue}{\textbf{1.000}} & \underline{0.974}
      & \textcolor{blue}{\textbf{1.000}} & \underline{0.934}
      & \textcolor{blue}{\textbf{0.910}} & \textcolor{blue}{\textbf{0.616}}
      & 0.338 & \underline{0.118} & 0.240 & 0.256 & 0.160 \\
    \addlinespace[1pt]
    \emph{Ablations:} & & & & & & & & & & & \\
    \quad $\ell_2$-norm (KATA-M2)
      & 1.000 & 0.836 & 0.982 & 0.792 & 0.658 & 0.224
      & 0.308 & 0.086 & 0.259 & 0.228 & 0.142 \\
    \quad short-conv, no RoPE (KATA-M2)
      & 0.956 & 0.960 & 0.948 & 0.926 & 0.760 & 0.530
      & 0.125 & 0.074 & 0.272 & 0.207 & 0.088 \\
    \bottomrule
  \end{tabular}
  \caption{\textbf{In-context recall of pretrained 340M language models.}
  RULER single-needle accuracy at 1K and 2K (left; 500 samples per length) and
  Arora'24 cloze accuracy (right), both under normalized substring matching.
  The best and second-best non-softmax primary methods in each column are shown in
  bold blue and underlined, respectively; diagnostic ablations are unranked.}
  \label{tab:niah}
  \label{tab:qa}
\end{table*}

\Cref{tab:niah} places softmax attention at the recall ceiling ($\ge 0.93$ on every task), while
Gated DeltaNet degrades sharply as needle entropy grows: word, number, and UUID recall fall
$0.996\!\to\!0.384\!\to\!0.004$ at 1K. The PSD maps retain much more high-entropy signal. KATA-M1
remains strong on word and number recall, but its UUID accuracy falls from $0.758$ at 1K to $0.458$
at 2K; KATA-M2 reaches $0.910$ and $0.616$, respectively.

At equal head width, the geometric packing argument favors KATA-M1's full 64-dimensional factor for
pure associative recall. Its recurrent state has $266{,}240$ entries, compared with $135{,}168$ for
KATA-M2, so KATA-M2 obtains the stronger hard-needle result with approximately half the state. We
interpret this reversal as a learnability effect: KATA-M1 learns one 64-dimensional spherical
address factor, while KATA-M2 learns a product of two 32-dimensional factors. Realized recall
therefore depends on both the available feature geometry and how readily the projections learn to
use it.

The delta erase provides a complementary gain. DeltaKATA-M1 raises UUID recall from $0.758$ to
$0.908$ at 1K and from $0.458$ to $0.584$ at 2K without changing the KATA-M1 feature width, approaching
KATA-M2's $0.910$ and $0.616$. Two ablations on KATA-M2 isolate the remaining design choices:
RMSNorm on $\vq,\vk$ improves 2K UUID recall from $0.224$ to $0.616$, and replacing RoPE with a
GDN-style short convolution preserves much of the in-window recall.

Length extrapolation exhibits a different ordering. Gated DeltaNet's multiplicative recurrent gate
generalizes more reliably beyond $2048$ tokens and leads the OOD-length evaluations. Its finite
address capacity remains visible on the hard needles within the training horizon: UUID recall is
$0.004$ at both 1K and 2K despite strong word recall. Strong OOD length generalization and
hard-needle recall therefore expose complementary properties of the mixer.

\subsection{In-context recall on the Based cloze suite}
\label{app:qa}
Complementing the controlled NIAH probe (\Cref{tab:niah}), we also run the pretrained 340M models on
the Based in-context recall suite~\citep{arora24based}, comprising SWDE, FDA, SQuAD,
TriviaQA, and NQ. Each task provides a document in context and a cloze query, scored by normalized
substring matching on the full dataset. \Cref{tab:qa} reports the results.

SWDE is closest to a direct key--value association. KATA-M1 reaches $0.375$, the best fixed-state
result, consistent with the pure-recall advantage of its larger feature space. FDA and SQuAD require
contextual interpretation and selection among correlated candidates. In this regime, KATA-M1 falls
to $0.061$ and $0.041$, while KATA-M2 reaches $0.118$ and $0.240$ and DeltaKATA-M1 reaches $0.148$
and $0.353$. The reversal supports the same learnability interpretation: the single
64-dimensional spherical address factor is harder to train for contextual routing than the product
of two 32-dimensional factors. KATA-M1's weaker pretraining perplexity relative to softmax attention,
Gated DeltaNet, and DeltaKATA-M1 compounds this difficulty. KATA-M2 improves contextual selection
despite its higher perplexity, indicating that factorization and general fluency contribute
separately.

Normalized substring matching also rewards verbose generations, so KATA-M1's terse answers are
under-credited relative to models that surface several candidate spans. We therefore treat NIAH as
the primary capacity result and the Based suite as a joint probe of retrieval, contextual routing,
and generation style.

\begin{table}[ht]
    \centering\footnotesize
  \begin{tabular}{l|ccccc}
    \toprule
    Method               & SWDE          & FDA  & SQuAD          & TriviaQA & NQ$^{\dagger}$ \\
    \midrule
    Softmax attn.        & 6.5           & 41.1 & 54.2           & 39.5     & 258 \\
    Gated DeltaNet       & 9.2           & 27.2 & 131.8          & 43.5     & 258 \\
    DeltaKATA-M1  & 4.5           & 40.9 & \textbf{202.2} & 46.3     & 258 \\
    KATA-M1   & 6.5           & 35.5 & \textbf{10.9}  & 72.8     & 258 \\
    KATA-M2   & 7.0           & 35.5 & 182.8          & 30.4     & 258 \\
    \bottomrule
  \end{tabular}
    \caption{\textbf{Average generation length.} Tokens averaged over 20 examples per task.
    $^{\dagger}$256-token cap on NQ.}
    \label{tab:qalen}
\end{table}

\begin{table}[ht]
    \centering\footnotesize
  \begin{tabular}{@{}l@{\quad}p{0.74\linewidth}@{}}
    \toprule
    \multicolumn{2}{@{}l}{\textbf{SQuAD}: cloze ``\dots the NFL team that represented the AFC at Super Bowl~50 was the''\ \ \textbf{[gold: Denver Broncos]}}\\
    \midrule
    Softmax          & New England Patriots. The 2016 NFL season was the 10th season in the NFL\dots\\
    Gated DeltaNet   & New England Patriots. The Patriots won the Super Bowl 50 with a record of 10--1\dots\\
    DeltaKATA-M1     & \textbf{Denver Broncos.} The Broncos were led by quarterback Tom Brady, who was named\dots\\
    KATA-M1 & New England Patriots.\\
    KATA-M2 & San Francisco 49ers, who defeated the New England Patriots 24--10\dots\\
    \midrule
    \multicolumn{2}{@{}l}{\textbf{SWDE} (structured copy): cloze ``\dots category: Feature\ \ year:''\ \ \textbf{[gold: 1983]}}\\
    \midrule
    all five models  & \texttt{1983}\\
    \bottomrule
  \end{tabular}
    \caption{\textbf{Example completions.} One prompt each from SQuAD and SWDE.}
    \label{tab:qaex}
\end{table}

\paragraph{Verbosity and fluency.} Because normalized substring matching credits any generation that includes the
gold span, it rewards length. \Cref{tab:qalen} reports the average generated length per model: on SQuAD
it ranges from $\sim\!11$ tokens (KATA-M1, which commits to a single terse answer) to
$\sim\!202$ (DeltaKATA-M1), an $18.6\times$ spread that mirrors the SQuAD accuracies. \Cref{tab:qaex} shows
example completions on a shared prompt: every model is fluent, but the terse KATA-M1 model is
penalized whenever its short answer is wrong, whereas more verbose models surface the gold span by
volume. On the copy-style SWDE task the effect vanishes because all five emit the same short field value.
This indicates that the SQuAD/FDA ordering reflects generation style as much as retrieval, a further
reason we take NIAH (short, exact needles) as the primary recall measure.

\subsection{Training compute}
\label{ssec:training-compute}
\begin{table}[H]
  \centering
  \footnotesize
  \setlength{\tabcolsep}{3.5pt}
  \begin{tabular}{lcrr}
    \toprule
    Run & GPU & \multicolumn{1}{c}{k tok/s/GPU} & \multicolumn{1}{c}{GPU-hours} \\
    \midrule
    Softmax attention & H100 & 188 & 22.2 \\
    Gated DeltaNet & H100 & 144 & 28.9 \\
    KATA-M1 (Flash) & H100 & 190 & 21.9 \\
    KATA-M2 (Flash) & H100 & 190 & 21.9 \\
    DeltaKATA-M1 & H100 & 77 & 54.1 \\
    \midrule
    MQAR + overwrite & A100 & {--} & $\sim\!4$ \\
    \bottomrule
  \end{tabular}
  \caption{\textbf{Training compute.} End-to-end throughput and accelerator time for each
  15B-token 340M pretraining run; the final row aggregates all MQAR and overwrite experiments.}
  \label{tab:training-times}
\end{table}

\Cref{tab:training-times} converts the measured end-to-end throughput into accelerator time using
$15\times10^9/(r\times3600)$ GPU-hours for throughput $r$ in tokens/s/GPU. Under near-linear
four-GPU scaling, the H100 totals correspond to approximately $5.5$ hours for softmax attention,
$7.2$ hours for Gated DeltaNet, $5.5$ hours for either Flash KATA model, and $13.5$ hours for
DeltaKATA-M1. The softmax run reaches approximately $760$k tokens/s across four H100 GPUs.

Kernel-level attention gains are amortized across the full model, whose runtime also includes the
MLPs, projections, RMSNorm, and optimizer work. Where applicable, we use FLA's fused recurrent
kernels~\citep{yang24fla}, including the Gated DeltaNet and delta-rule recurrences, to reduce launch
and intermediate-materialization overhead. The complete MQAR and overwrite benchmark suite is much
smaller, requiring approximately $4$ A100 GPU-hours in total.

\section{Related Work} \label{sec:related}
We organize prior work along the four axes that \cref{sec:rkhs,sec:expressivity,sec:hardware}
unify: feature geometry, state recurrence, capacity diagnostics, and hardware mapping. For each, we
indicate where existing linear-attention designs fall short.

\paragraph{Feature geometry.} \citet{katharopoulos20transformers} replace softmax with
$\psi(\vx)=\mathrm{ELU}(\vx)+1$ and exploit the recurrence
$\mS_t=\mS_{t-1}+\psi(\vk_t)\vv_t^\top$ to recurse in $\gO(n_\psi d_v)$.
Performers~\citep{choromanski21performers}, RFA~\citep{peng21rfa}, cosFormer~\citep{qin22cosformer},
and the fast-weight view of \citet{schlag21linear} use random, trigonometric, or other finite
features to approximate $\exp(\lrangle{\vq,\vk})$. When nonnegative attention weights are required,
such constructions use orthant-valued maps. The orthant is the simplicial member of the real
self-dual cones admitted by the Koecher--Vinberg classification
~\citep{koecher57positivitatsbereiche,vinberg63homogeneous,faraut94analysis}. The Welch bound
(\Cref{thm:welchbound}) then gives
$\mu(\R_+^{n_\psi},\seql)^2\ge(\seql-n_\psi)/(n_\psi(\seql-1))$. At $n_\psi=128$, requiring
$\mu\le10^{-2}$ limits the dictionary to $\seql\lesssim130$ keys. These works do not analyze the
orthant restriction as a capacity bottleneck.

\paragraph{State recurrence.} S4~\citep{gu22efficiently}, RWKV~\citep{peng23rwkv}, RetNet~\citep{sun23retentive}, and
Mamba~\citep{gu24mamba} generalize the recurrence to a data-dependent transition
$\mS_t=\mA_t\mS_{t-1}+\vb_t\vc_t^\top$. The State-Space Duality of Mamba2~\citep{dao24mamba2}
restricts $\mA_t=\alpha_t\mI$, exposing the algebraic equivalence with causal linear attention but
compressing per-token selectivity into a single scalar; the resulting decay is empirically too
coarse on retention probes (S-NIAH degrades past $2$K tokens~\citep{yang25gateddeltanet}).
DeltaNet~\citep{yang24deltanet} treats $\mS_t$ as a fast-weight matrix updated by online ridge
regression, and Gated DeltaNet~\citep{yang25gateddeltanet} adds an explicit forget gate:
\begin{equation}
  \mS_t = \alpha_t(\mI-\beta_t\vk_t\vk_t^\top)\mS_{t-1} + \beta_t\vk_t\vv_t^\top,
  \quad \alpha_t,\beta_t\in\R.
  \label{eq:related-delta}
\end{equation}
\Cref{prop:output-gate} shows that normalized linear attention produces the parameter-free convex
coefficient $\alpha_t=D_{t-1}/D_t$ at readout while leaving the stored state unchanged. Recurrent
gates act directly on memory and therefore provide a distinct control point. When the raw-key
substrate in \cref{eq:related-delta} is constrained to the nonnegative geometry studied here, its
Lorentz lift meets the $1/2$-interference Rankin wall of \cref{thm:cap-lor}
~\citep{rankin55closest}, giving an $\gO(d)$ address regime. Content erase and learned decay can
improve overwrite and retention without enlarging the raw feature space. These models tune the
recurrence while leaving $\psi$ implicit. Channel-wise refinements~\citep{qiu25gated} and
test-time-regression variants~\citep{behrouz24titans,zuo25localla} operate on the same raw-key
substrate; their recurrence mechanisms are complementary to the feature geometry studied here.

\paragraph{Capacity diagnostics.} Zoology~\citep{arora23zoology} and MQAR/needle-in-a-haystack benchmarks expose the recall--memory
trade-off~\citep{jelassi24repeat,du25mom}. Based~\citep{arora24based} approximates
$\exp(\lrangle{\vq,\vk})$ by a degree-$2$ Taylor map paired with sliding-window attention;
LoLA~\citep{mcdermott25lola}, the Associative Memory layer~\citep{krotov25mam}, and the
GatedDeltaNet-H1/H2 stacks~\citep{yang25gateddeltanet} concede the pure linear path and reintroduce
a small KV cache. Modern Hopfield
networks~\citep{krotov16dense,demircigil17model,ramsauer21hopfield,zhong25transformersdam}
establish that exponential energy admits exponentially many stable patterns, and \cref{app:dam}
places KATA's rank-one PSD map at degree $2$ in this hierarchy with $\gO(d^2)$ feature state.
The map $\psi^{\mathrm S}(\vu)=\mathrm{vec}(\vu\vu^\top)$ squares the inner product:
$\lrangle{\psi^{\mathrm S}(\vu_i),\psi^{\mathrm S}(\vu_j)}=\lrangle{\vu_i,\vu_j}^2$.
Below the Welch threshold $\eps^\star=1/p$, \Cref{thm:welchbound} gives a finite rational ceiling.
For any fixed $\eps\in(0,1)$, once $p>1/\eps$, the greedy spherical-cap construction packs
$\seql\ge\tfrac12(1-\eps)^{-(p-1)/2}$ near-orthogonal keys in
$p(p+1)/2$ feature coordinates, which grows exponentially in $p$ (\cref{thm:cap-spd}).
\paragraph{Hardware mapping.} Chunkwise parallel training is the standard scaffold for linear-time
recurrent attention, with the inter-chunk state propagated by a sequential prefix sum in, for
example, the FLA kernels~\citep{yang24fla}. For Gated DeltaNet, the non-commutative Householder
transition in \cref{eq:related-delta} requires a $C\!\times\!C$ triangular solve per chunk under the
WY representation~\citep{schreiber89compactwy} and materializes per-token state intermediates in HBM.
Mamba2's data-dependent transition instead requires a multiplicative selective scan rather than a
vector-space cumulative sum. An additive feature-state recurrence admits a parallel scan that lowers the
inter-chunk depth from $\Theta(N_C)$ to $\gO(\log N_C)$ (\Cref{ssec:treescan}). KATA combines
Koecher--Vinberg geometry with Welch--Rankin packing to certify exponential capacity at a fixed
tolerable interference using a finite, hardware-aligned feature map.

\section{Discussion and Limitations} \label{sec:discussion}
The results separate three aspects of associative attention that are often grouped under
\emph{capacity}: the geometry available for storing bindings, the ability of training to realize that
geometry, and the contextual fluency needed to construct the correct association from language.
KATA provides a theory and hardware realization for the first, exposes architectural controls for
the second, and tests the third through 340M-parameter pretraining.

\paragraph{Geometric capacity.}
For normalized query and key directions, associative capacity is a packing problem. Within the
ordinary real symmetric-cone family, \Cref{sec:rkhs,sec:expressivity} identifies the available
nonnegative geometries and bounds their attainable interference. The rank-one PSD map squares
directional overlap and expands a $p$-dimensional address into $p(p+1)/2$ features. KATA-M1 realizes
the full map, while KATA-M$g$ and KATA-$\Sigma g$ reduce the recurrent state through block
factorization. MQAR and NIAH validate the resulting capacity advantage over raw-key recurrences, and
the same pSNR framework explains temperature sharpening in softmax. These bounds describe an
achievable ceiling under the real, unit-direction, degree-two model.

\paragraph{Learnability.}
Training determines how much of that ceiling is realized, so empirical associative-recall performance reflects the quality of the packing learned by the model. KATA-M1 has the larger geometric packing budget
and a $266{,}240$-entry state, yet KATA-M2 obtains stronger hard-needle recall with $135{,}168$
entries. In this regime, the model learns two 32-dimensional factors more effectively than one
64-dimensional factor. Normalization and recurrence design provide additional controls: RMSNorm
improves the learned KATA-M2 geometry, the delta rule replaces stale content at repeated addresses,
and multiplicative gating supplies learned positional weighting distinct from RoPE. Gated DeltaNet's
length extrapolation shows that such gating can generalize beyond the RoPE training horizon.

\paragraph{Contextual fluency.}
Direct recall begins after the model has formed the relevant query, key, and value. FDA and SQuAD
also require context interpretation and selection among correlated candidates. KATA-M1 is strong on
NIAH and direct-copy SWDE but degrades on these contextual tasks; KATA-M2 and DeltaKATA-M1 recover
part of the gap. Pretraining perplexity contributes separately, since KATA-M2 improves contextual
selection despite higher perplexity. The combined synthetic, exact-string, and cloze evaluations
therefore distinguish storage capacity from the broader language-modeling behavior of the mixer.

\paragraph{Future work.}
Four directions follow directly. First, packing theory and constructive dictionaries should be
extended to complex spaces, where spherical codes and ETFs may offer more capacity at equal
dimension. Second, GatedKATA and GatedDeltaKATA should be pretrained systematically, while
DeltaKATA should be evaluated beyond M1; all three require dedicated fused GPU kernels. Third, the
learnability of the full KATA-M1 geometry should be tested beyond $d_\text{head}=64$, together with kernels that make its quadratic
feature state practical at wider heads. Finally, tensor-train factorizations
~\citep{oseledets11tt} offer a path to higher even-order outer-product features without materializing
the full tensor state. These studies will determine whether the geometric gains persist as feature
order, model scale, and contextual demands increase.

\section*{Acknowledgments}
We thank the Research Computing staff at the University of Colorado Boulder for providing access to
GH200 nodes and for their prompt support in diagnosing and resolving system issues. We also thank
Modal for free-tier compute credits that accelerated kernel tuning and development. We are grateful
to the GPU MODE community, whose tutorials, competitions, forums, and open technical discussions
created a welcoming environment for learning and mastering GPU programming. The answers shared by
community members on those forums were crucial to developing the kernels in this work.

\bibliographystyle{plainnat}
\bibliography{references}

\newpage
\appendix
\thispagestyle{empty}
\crefalias{section}{appendix}
\crefalias{subsection}{appendix}
\crefalias{subsubsection}{appendix}
\section{Reader's guide to the geometry}
\label{app:reader-guide}

This appendix provides a short map from the attention equations to the geometric claims in the
main text. We reserve $d$ for the raw query/key dimension, $d_v$ for the value dimension, and
$n_\psi$ for a generic feature dimension. For PSD features, $p$ is the matrix side length and
$n\defeq p(p+1)/2$ is the symmetric ambient dimension, so the full PSD map has $n_\psi=n$.
The sequence length is $\seql$, the chunk size is $C$, and
$N_C\defeq\lceil \seql/C\rceil$ is the number of chunks. Packing statements define their dictionary
size locally.

\paragraph{Why recall becomes a packing problem.}
Suppose a query $\vq$ is meant to retrieve the binding at key $\vk_j$ from the additive state
$\mS=\sum_i \psi(\vk_i)\vv_i^\top$. The numerator of the readout is:
\[
  \psi(\vq)^\top \mS
  =
  \lrangle{\psi(\vq),\psi(\vk_j)}\vv_j^\top
  +
  \sum_{i\ne j}\lrangle{\psi(\vq),\psi(\vk_i)}\vv_i^\top .
\]
After normalizing features so the matched coefficient is one, every wrong value is weighted by a
feature inner product. A fixed recurrent state therefore stores many bindings well only when the
stored feature vectors are nearly orthogonal. This is the interference $\mu(\gK,T)$ in
\Cref{eq:mu-def}: it asks how many unit feature vectors fit in the allowed cone with all pairwise
interference below a target tolerance.

\paragraph{Why cones enter.}
Normalized linear attention needs nonnegative weights in the denominator and numerator. A self-dual
cone $\gK=\gK^*$ gives a coordinate-free certificate: if $\psi(\vx),\psi(\vy)\in\gK$, then
$\lrangle{\psi(\vx),\psi(\vy)}\ge0$. This is weaker than strict coordinatewise positivity. Boundary
points are allowed; small $\eps$ terms are numerical safeguards against a zero denominator, not
part of the capacity theorem.

\paragraph{Where the denominator gate appears.}
For a fixed query, write $D_t=\psi(\vq_t)^\top\mZ_t=D_{t-1}+\lrangle{\psi(\vq_t),\psi(\vk_t)}$. If
$D_t>0$, the normalized readout is a convex interpolation between the previous normalized readout
and the new value, with weights $D_{t-1}/D_t$ and $\lrangle{\psi(\vq_t),\psi(\vk_t)}/D_t$. This is
the parameter-free gate proved formally in \Cref{prop:output-gate}; it comes from normalization,
not from an additional state transition.

\paragraph{What the three cone factors do.}
The orthant has $d$ coordinate rays that are exactly orthogonal, but beyond those rays it inherits
the ordinary Welch interference floor. The positive orthant has small solid angle, but a worst-case
packing can live on sparse boundary faces, so volume alone is not a universal capacity obstruction.
The Lorentz lift of a linear key uses $\psi_\mathrm{Lor}(\vk)=2^{-1/2}(\vk,1)$ and turns
$z=\lrangle{\vq,\vk}$ into $(1+z)/2$; this adds the $1/2$ offset, so pushing below that offset
forces the raw keys to be mutually obtuse and invokes Rankin. The PSD rank-one map
$\psi_\mathrm{PSD}(\vu)=\vu\vu^\top$ turns $z$ into $z^2$, so antipodal and orthogonal-ish raw keys
identifies antipodal raw vectors as the same line, while near-orthogonal raw lines become useful packing directions.

\paragraph{How softmax fits.}
The exponential kernel is an infinite direct sum of tensor degrees. \Cref{thm:softmax-cone-decomp}
shows the cone interpretation of those degrees: even monomials are rank-one PSD rays on tensorized
features, while signed odd monomials become nonnegative only when paired with enough even mass into
Lorentz-gated PSD atoms. KATA keeps the first capacity-changing PSD degree and implements it with a
fixed recurrent state.

\subsection{Kernel methods}
\label{app:kernel_methods}

The exponential dot-product kernel admits the tensor expansion:
\begin{equation}
  \exp(\lrangle{\vx,\vy})
  = \sum_{r=0}^{\infty}\frac{\lrangle{\vx,\vy}^r}{r!}
  = \sum_{r=0}^{\infty}
  \left\langle \frac{\vx^{\otimes r}}{\sqrt{r!}},
  \frac{\vy^{\otimes r}}{\sqrt{r!}}\right\rangle .
\end{equation}
Thus the canonical RKHS feature map is the infinite direct sum
$\phi(\vx)=\bigoplus_{r\ge 0}\vx^{\otimes r}/\sqrt{r!}$. Linear attention replaces this infinite
feature map by a finite $\psi:\R^d\to\R^{n_\psi}$.

\begin{lemma}[Norm of the exponential feature map on the sphere]
  \label{lem:exp-norm}
  For $\tau>0$, let
  $K_\tau(\vq,\vk)=\exp(\lrangle{\vq,\vk}/\tau)$ have canonical feature map:
  \[
    \phi_\tau(\vx)
      =\bigoplus_{r\ge0}\frac{\vx^{\otimes r}}{\sqrt{r!\tau^r}}.
  \]
  Then:
  \[
    \norm{\phi_\tau(\vx)}_\gH^2
      =K_\tau(\vx,\vx)
      =\exp(\norm{\vx}^2/\tau).
  \]
  In particular, on $\sphere{d-1}$ the squared feature norm is the constant $e^{1/\tau}$
  and the feature norm is $e^{1/(2\tau)}$.
\end{lemma}
\begin{proof}
  The claim follows by evaluating the reproducing kernel on the diagonal:
  $\norm{\phi_\tau(\vx)}_\gH^2
   =\lrangle{\phi_\tau(\vx),\phi_\tau(\vx)}_\gH
   =K_\tau(\vx,\vx)$.
\end{proof}

Hence unit-normalizing $\vq,\vk$ in input space maps the sphere into a sphere of radius
$e^{1/(2\tau)}$ in the exact RKHS, matching the cone-restricted fixed-norm setting used throughout
the main text up to a positive scalar.

\paragraph{RMSNorm effect}
\label{app:rms-temperature}

RMSNorm normalizes the input direction and then applies a learned per-channel scale. Its output is
therefore not unit norm; nevertheless, its query--key inner product can be represented as a scalar
multiple of an inner product between unit vectors in one higher dimension. Ignoring the stabilizing
$\epsilon$, let $c_d>0$ absorb the fixed conversion from RMS normalization to $\ell_2$
normalization, together with any fixed attention-logit scaling, and define:
\[
  \widehat{\vq}=\frac{\hvq}{\norm{\hvq}},
  \qquad
  \widehat{\vk}=\frac{\hvk}{\norm{\hvk}}.
\]
Define the diagonal gain matrix:
\[
  \mD_\gamma
  \defeq
  c_d\,\operatorname{diag}(\vgamma_q\odot\vgamma_k).
\]
Then:
\begin{align}
  \ell(\hvq,\hvk)
    &=
    \widehat{\vq}^{\top}\mD_\gamma\widehat{\vk} \\
    &=
    \lrangle{\widehat{\vq},\bar{\vk}},
  \qquad
  \bar{\vk}
    \defeq
    \mD_\gamma\widehat{\vk}.
  \label{eq:rms-fused-key}
\end{align}
Choose any $M>0$ satisfying $M\ge\norm{\bar{\vk}}^2$ for every key under consideration. For a
finite dictionary one may take $M=\max_j\norm{\bar{\vk}_j}^2$; a context-independent choice for $\mD_\gamma\neq0$ is:
\[
  M=\norm{\mD_\gamma}_{\mathrm{op}}^2.
\]
Define:
\[
  \widetilde{\vq}=(\widehat{\vq},0),
  \qquad
  \widetilde{\vk}
    =
    \frac{1}{\sqrt M}
    \left(\bar{\vk},\sqrt{M-\norm{\bar{\vk}}^2}\right).
\]
Both $\widetilde{\vq}$ and $\widetilde{\vk}$ are unit vectors in $\R^{d+1}$, and:
\begin{equation}
  \ell(\hvq,\hvk)
  =
  \sqrt M\,
  \lrangle{\widetilde{\vq},\widetilde{\vk}}.
  \label{eq:rms-temperature-lift}
\end{equation}
Thus the RMSNorm-scaled logit is exactly a scalar inverse-temperature factor $\sqrt M$ multiplying
a cosine similarity in the lifted space. The diagonal gain matrix is absorbed into the lifted key geometry.

\subsection{Symmetric cones and Jordan algebras}
\label{app:cones}

A subset $\gK\subset\R^n$ is a \emph{cone} if $s\vv\in\gK$ for all $s>0$, $\vv\in\gK$. It is
\emph{convex} if $\vu+\vv\in\gK$ whenever $\vu,\vv\in\gK$. A closed convex cone is \emph{pointed}
if $\gK\cap(-\gK)=\{0\}$ and \emph{full-dimensional} if its interior is non-empty.

The \emph{dual cone} of $\gK\subset\R^n$ under the Euclidean inner product is $\gK^* = \{\vy\in\R^n
  : \lrangle{\vx,\vy}\ge 0\;\forall\vx\in\gK\}$. The cone is \emph{self-dual} if $\gK = \gK^*$. Its
\emph{automorphism group} is $\mathrm{Aut}(\gK) = \{\mA\in\mathrm{GL}(n,\R) : \mA(\gK) = \gK\}$,
and $\gK$ is \emph{homogeneous} if $\mathrm{Aut}(\gK)$ acts transitively on the interior
$\operatorname{int}(\gK)$.

The \emph{Koecher--Vinberg
  theorem}~\citep{koecher57positivitatsbereiche,vinberg63homogeneous,faraut94analysis} identifies
real symmetric cones with cones of squares in formally real (Euclidean) Jordan algebras, and
classifies the irreducible cases as: $\R_+$ (one-dimensional Jordan algebra), $\gL^m_+$ for $m\ge
  3$ (spin factor), $\sS^n_+$ for $n\ge 2$ (real symmetric matrices), $\gH^n_+(\mathbb{C})$ (complex
Hermitian, $n\ge 2$), $\gH^n_+(\mathbb{H})$ (quaternionic Hermitian, $n\ge 2$), and
$\gH^3_+(\mathbb{O})$ (Albert algebra, exceptional). Our analysis restricts to the first three
(real-coordinate Jordan algebras); the others remain valid symmetric cones but require complex-,
quaternionic-, or octonionic-valued feature maps.

\subsection{Fundamental geometric bounds}
\label{app:geom-bounds}

We collect the four bounds that govern packing on the cone-restricted sphere. Throughout, $\seql$ is
the number of unit vectors and $D$ is the ambient dimension.

\paragraph{Welch bound.} For any $\seql$ unit vectors $\vu_1,\dots,\vu_{\seql}\in\mathbb{S}^{D-1}\subset\R^D$:
\[
  \max_{i\neq j} \lrangle{\vu_i,\vu_j}^2 \;\ge\; \frac{\seql-D}{D(\seql-1)}.
\]
Inverting for $\seql$ at squared interference $\eps$ gives $\seql \le D(1-\eps)/(1-D\eps)$ for $\eps<1/D$.
As $\seql\to\infty$, the floor approaches $1/D$. The Welch bound is achieved with equality by
Equiangular Tight Frames (ETFs).

\paragraph{Gerzon bound.} An ETF on $\seql$ vectors in $\R^D$ requires
$\seql\le D(D+1)/2$; above this ceiling, equality in the Welch bound is impossible. Below the ceiling,
existence remains parameter-dependent, and real ETFs occur only for special pairs $(D,\seql)$. Over
$\mathbb C^D$, the corresponding ceiling is $D^2$, although complex ETF existence is also
incomplete~\citep{fickus16etf}.

\paragraph{Rankin bound.} If $\seql$ unit vectors in $\R^D$ have strictly negative pairwise inner
products, then $\seql\le D+1$. At the boundary where the inner products are merely nonpositive, the
cross-polytope gives the sharp bound $\seql\le2D$. The Lorentz condition in \Cref{thm:cap-lor} is strict:
for interference below $1/2$, its spatial components lie in $\R^{d-1}$ and are mutually obtuse, so
$\seql\le d$.

\paragraph{Alon's combinatorial bound.} Alon's rank bound~\citep{alon09rank} states that for any $\seql$ unit vectors with
$\max_{i\ne j}|\lrangle{\vu_i,\vu_j}|\le\mu$:
\[
  D
  = \Omega\!\left(
      \min\!\left\{\seql,\frac{\log \seql}{\mu^2\log(1/\mu)}\right\}
    \right)
\]
This universal lower bound leaves only a $\log(1/\mu)$ factor between the best known necessary
dimension and the sufficient $D=\gO(\mu^{-2}\log \seql)$ scaling of random spherical codes.

\paragraph{Finite-size comparison.} Alon's bound hides an unspecified universal constant, so we do
not assign it a numerical value. For $\seql=10^5$ and $\mu=0.05$, the Welch bound requires $D\ge399$,
while the simple all-pairs random-code estimate $D\approx4\log \seql/\mu^2$ gives $D\approx18{,}421$.
The deterministic MUB-derived and DeVore dictionaries below require $D=802$ and $D=2{,}209$,
respectively.

\subsection{Dense associative memory hierarchy}
\label{app:dam}

The cone feature degrees parallel the Krotov--Hopfield Dense Associative Memory (DAM) hierarchy.
DAM capacity is defined through an energy, update rule, and stability criterion, whereas our
$\eps$-capacity counts address directions at fixed interference. The comparison below concerns the
tensor-power degree of $F(z)=\lrangle{\psi(\vq),\psi(\vk)}$ for
$z=\lrangle{\vq,\vk}$ and does not identify the two capacity notions.

\begin{itemize}
  \item \textbf{Unnormalized linear attention:} $F_\mathrm{Lin}(z)=z$, degree-1 (classical
        Hopfield), using the raw signed key kernel. Cone nonnegativity is absent.
  \item \textbf{Lorentz lift of the linear kernel:} $F_\mathrm{Lor}(z)=\tfrac12+\tfrac12z$.
        This is the smallest nonnegative cone lift of the degree-1 substrate. Below interference
        $1/2$, its geometric address capacity is $\gO(p)$ by the Rankin wall of \Cref{thm:cap-lor}.
  \item \textbf{Square-diagonal PSD with rank-one rays:} $F_\mathrm{PSD}(z)\propto z^2$, degree-2
        (Krotov--Hopfield dense memory). For fixed interference and $p>1/\eps$, the spherical-cap
        construction of \Cref{thm:cap-spd} gives exponentially many address directions. Under the
        isotropic-value pSNR criterion, aggregate retrieval instead has the $\Theta(p^2)$ threshold
        of \Cref{thm:snr}.
  \item \textbf{Softmax:} $F_\mathrm{soft}(z) = e^z = \sum_{r\ge 0} z^r/r!$, infinite degree.
        \Cref{thm:softmax-cone-decomp} shows that its even Taylor degrees are rank-one PSD rays on
        tensor powers, while its signed odd degrees become positive only after pairing with even
        mass into Lorentz-gated PSD atoms. The kernel requires the infinite RKHS feature
        $\phi(\vx) = \bigoplus_r \vx^{\otimes r}/\sqrt{r!}$.
\end{itemize}

KATA's degree-2 map uses $\gO(p^2)$ feature coordinates without adding $\gO(p^2d_v)$ model
parameters. The quadratic kernel rematerializes the $p\to p^2$ feature expansion on chip, while
the linear-state kernel stores the expanded recurrent state in HBM. Tensor-train
factorizations~\citep{oseledets11tt} are a possible route to higher feature orders without
materializing the full tensor; enforcing nonnegativity under such truncations remains open.

\section{Cone classification and normalization proofs}
\label{app:cone-proofs}

\subsection{\texorpdfstring{Proof of \Cref{prop:output-gate} (output-level gating)}{Proof of output-level gating}}

Let $D_{t-1}=\psi(\vq_t)^\top\mZ_{t-1}$ and $c_t=\psi(\vq_t)^\top\psi(\vk_t)$. If $D_{t-1}=0$, then
the nonnegative summands in $D_{t-1}$ are all zero, so $\psi(\vq_t)^\top\mS_{t-1}=0$ as well. Since
$D_t=D_{t-1}+c_t>0$, we have $c_t>0$ and the normalized readout is exactly $\vv_t^\top$.

Now suppose $D_{t-1}>0$. Splitting numerator and denominator of $\vz_t$:
\[
  \vz_t =
  \frac{\psi(\vq_t)^\top \mS_{t-1}
    + \psi(\vq_t)^\top\psi(\vk_t)\vv_t^\top}
  {\psi(\vq_t)^\top \mZ_{t-1}+\psi(\vq_t)^\top\psi(\vk_t)}.
\]
Multiplying and dividing the first numerator term by $\psi(\vq_t)^\top\mZ_{t-1}$ yields:
\[
  \vz_t =
  \frac{\psi(\vq_t)^\top\mZ_{t-1}}{\psi(\vq_t)^\top\mZ_t}\,
  \bar\vz_{t|t-1}
  + \frac{\psi(\vq_t)^\top\psi(\vk_t)}{\psi(\vq_t)^\top\mZ_t}\,
  \vv_t^\top .
\]
The two numerators sum to $\psi(\vq_t)^\top\mZ_t$, so $\alpha_t(\vq_t)+\beta_t(\vq_t)=1$.
Nonnegativity follows from \Cref{ass:pos}, because $D_{t-1}$ and $c_t$ are sums of nonnegative
feature inner products.

\subsection{Canonical-cone classification}

Let $\gK$ be the closed, pointed, full-dimensional self-dual homogeneous cone that contains
$\operatorname{im}\psi$. Self-duality gives $\gK=\gK^*$, so for any two features
$\psi(\vx),\psi(\vy)\in\gK$:
\[
  \lrangle{\psi(\vx),\psi(\vy)}\ge0,
\]
which certifies \Cref{ass:pos}. By the Koecher--Vinberg theorem, $\gK$ is the cone of squares of a
Euclidean Jordan algebra and decomposes uniquely as a Cartesian product of irreducible symmetric
cones. Restricting to the ordinary real families considered here leaves $\R_+$, Lorentz cones, and
real PSD cones. Products of one-dimensional $\R_+$ factors are orthants. Thus every cone in our
scope is represented as a product of orthant, Lorentz, and PSD components. \Cref{ass:iso} constrains how
rotations of the input act within this product, but does not introduce additional real irreducible
cone factors.

\section{Packing and SNR proofs}
\label{app:packing}

\subsection{\texorpdfstring{Proof of \Cref{thm:welchbound} (orthant Welch limit)}{Proof of orthant Welch limit}}

The $d$ standard basis vectors lie in $\R^d_+$ and are pairwise orthogonal, so zero interference is
possible for $N=d$. For $N>d$, apply the Welch bound to any $N$ unit vectors
$\bm{\psi}_1,\dots,\bm{\psi}_N\in\R^d_+$:
\[
  \max_{i\ne j}\lrangle{\bm{\psi}_i,\bm{\psi}_j}^2
  \ge \frac{N-d}{d(N-1)}.
\]
Because orthant inner products are nonnegative, this is also a lower bound on the squared interference
$\mu^2$. If $\mu\le\eps$ and $\eps<1/\sqrt d$, rearranging $\eps^2\ge(N-d)/(d(N-1))$ gives $N\le
  d(1-\eps^2)/(1-d\eps^2)$.

\subsection{\texorpdfstring{Proof of \Cref{thm:cap-lor} (Lorentz interference wall)}{Proof of Lorentz interference wall}}

\begin{proof}
  A unit extreme ray of the closure of $\gL^d_{+}$ has the form
  $\psi=2^{-1/2}(\vy,1)$ with $\vy\in\mathbb{S}^{d-2}$. Hence:
  \[
    \lrangle{\psi_i,\psi_j}
      =\frac{1}{2}\lrangle{\vy_i,\vy_j}+\frac{1}{2}.
  \]
  If $\lrangle{\psi_i,\psi_j}\le\eps<1/2$, then
  $\lrangle{\vy_i,\vy_j}\le-\alpha$ with $\alpha=1-2\eps>0$. Therefore:
  \[
    0
    \le \norm{\sum_{i=1}^{N}\vy_i}_2^2
    = N+2\sum_{i<j}\lrangle{\vy_i,\vy_j}
    \le N-\alpha N(N-1),
  \]
  which gives $N\le1+1/\alpha$. Since the pairwise inner products are strictly negative,
  Rankin's strict spherical-code bound also gives $N\le d$ for vectors in
  $\R^{d-1}$~\citep{rankin55closest}. Combining the two bounds proves the claim.
\end{proof}

\begin{lemma}[One-sided spherical-cap bound]
  \label{lem:spherical-cap-bound}
  Let $p\ge2$, let $\vu\sim\mathrm{Unif}(\mathbb{S}^{p-1})$, and fix
  $\ve\in\mathbb{S}^{p-1}$. For every $a\in[0,1]$:
  \begin{equation}
    \Pr\!\left[\lrangle{\vu,\ve}\ge a\right]
    \le (1-a^2)^{(p-1)/2}.
    \label{eq:spherical-cap-bound}
  \end{equation}
\end{lemma}

\begin{proof}
  By rotation invariance, take $\ve$ to be the north pole and write
  $\theta=\arccos(a)\in[0,\pi/2]$. With $\sigma$ denoting normalized surface measure, the cap
  $C_a(\ve)=\{\vu:\lrangle{\vu,\ve}\ge a\}$ has measure:
  \[
    \sigma(C_a(\ve))
    = \frac{\int_0^\theta \sin^{p-2}\phi\,d\phi}
    {\int_0^\pi \sin^{p-2}\phi\,d\phi}
    = \frac12 I_{1-a^2}\!\left(\frac{p-1}{2},\frac12\right),
  \]
  where $I_x(\alpha,\beta)$ is the regularized incomplete beta function.

  It remains to upper-bound the first ratio. Set $n=p-2$ and $s=\sin\theta=\sqrt{1-a^2}$. If $s=0$,
  the claim is trivial. Otherwise substitute $\sin\phi=s\sin\psi$ in the numerator. Then:
  \[
    \int_0^\theta \sin^n\phi\,d\phi
    = s^{n+1}\int_0^{\pi/2}
    \frac{\sin^n\psi\cos\psi}
    {\sqrt{1-s^2\sin^2\psi}}\,d\psi .
  \]
  Since $0\le s\le1$, we have $\sqrt{1-s^2\sin^2\psi}\ge\sqrt{1-\sin^2\psi}=\cos\psi$ for
  $\psi\in[0,\pi/2]$. Hence:
  \[
    \int_0^\theta \sin^n\phi\,d\phi
    \le s^{n+1}\int_0^{\pi/2}\sin^n\psi\,d\psi
    \le s^{n+1}\int_0^\pi\sin^n\psi\,d\psi .
  \]
  Dividing by the denominator gives $\sigma(C_a(\ve))\le s^{n+1}=(1-a^2)^{(p-1)/2}$.
\end{proof}

\subsection{\texorpdfstring{Proof of \Cref{thm:cap-spd} (PSD exponential packing)}{Proof of PSD exponential packing}}

For $p\ge2$, the extreme rays of the PSD cone $\sS^p_{+}$ are rank-one matrices $\vu\vu^\top$ with
$\vu\in\mathbb{S}^{p-1}$. Their Frobenius inner product is:
\[
  \lrangle{\vu_i\vu_i^\top,\vu_j\vu_j^\top}_F
  = \Tr(\vu_i\vu_i^\top\vu_j\vu_j^\top)
  = \lrangle{\vu_i,\vu_j}^2 .
\]
Thus PSD interference at most $\eps$ is equivalent to $|\lrangle{\vu_i,\vu_j}|\le\sqrt{\eps}$. Let
$a=\sqrt{\eps}$ and let $\sigma$ denote normalized surface measure on $\mathbb{S}^{p-1}$.

Now choose lines greedily on $\mathbb{S}^{p-1}$: after selecting $\vu_1,\dots,\vu_m$, add any $\vu$
outside the antipodal caps $C_a(\vu_i)\cup C_a(-\vu_i)$. When the process is maximal, these
antipodal caps cover the sphere; otherwise an uncovered point could be added. By
\Cref{lem:spherical-cap-bound}, each selected line covers at most $2(1-\eps)^{(p-1)/2}$ surface
measure, so maximality implies:
\[
  1 \le 2m(1-\eps)^{(p-1)/2}
  \quad\Rightarrow\quad
  m \ge \frac12(1-\eps)^{-(p-1)/2}.
\]
Because each new line is chosen outside all previous antipodal caps, the selected lines satisfy
$|\lrangle{\vu_i,\vu_j}|<a$, hence they meet the stated $\eps$ interference bound after the
rank-one PSD map. Adding $\eta\mI$ to each rank-one matrix and renormalizing moves the construction
into the cone interior if a numerical margin is desired, and the pairwise Frobenius inner products
converge to the rank-one values as $\eta\downarrow0$.

\subsection{Rank-one PSD atoms minimize Haar-average interference}\label{app:rank-one-optimal}

The ray through $\vu\vu^\top$ is extreme. Indeed, if $\vu\vu^\top=\mA+\mB$ with $\mA,\mB\succeq0$,
then for every $\vw\perp\vu$, $0=\vw^\top\mA\vw+\vw^\top\mB\vw$. Positive semidefiniteness gives
$\mA\vw=\mB\vw=0$, so the ranges of $\mA$ and $\mB$ are contained in $\mathrm{span}\{\vu\}$. Hence
$\mA$ and $\mB$ are nonnegative multiples of $\vu\vu^\top$.

\begin{proposition}[Rank-one PSD atoms minimize Haar-average interference]
  \label{prop:spd-rank-one}
  Let $\mA,\mB\succeq 0$ have $\norm{\mA}_F=\norm{\mB}_F=1$ with eigendecompositions
  $\mA=\sum_r a_r\vu_r\vu_r^\top$ and $\mB=\sum_s b_s\vv_s\vv_s^\top$, $a_r,b_s\ge 0$. Then
  $\lrangle{\mA,\mB}_F=\sum_{r,s}a_r b_s\lrangle{\vu_r,\vv_s}^2\ge 0$, and under a Haar-random
  relative eigenbasis $\E\lrangle{\mA,\mB}_F=\Tr(\mA)\Tr(\mB)/p\ge 1/p$, with equality only for
  rank-one matrices.
\end{proposition}

Let $\mA=\sum_r a_r\vu_r\vu_r^\top$ and $\mB=\sum_s b_s\vv_s\vv_s^\top$ be spectral decompositions
with $a_r,b_s\ge0$. Then:
\[
  \lrangle{\mA,\mB}_F
  = \Tr\left(\sum_{r,s}a_r b_s\vu_r\vu_r^\top\vv_s\vv_s^\top\right)
  = \sum_{r,s}a_r b_s\lrangle{\vu_r,\vv_s}^2 .
\]
All coefficients are nonnegative, so high-rank PSD features average many rank-one interactions
without cancellation. If the relative eigenbasis is Haar-random, then
$\mathbb{E}\lrangle{\vu_r,\vv_s}^2=1/p$ for every pair $(r,s)$, hence:
\[
  \mathbb{E}\lrangle{\mA,\mB}_F
  = \frac{1}{p}\left(\sum_r a_r\right)\left(\sum_s b_s\right)
  = \frac{\Tr(\mA)\Tr(\mB)}{p}.
\]
Since $\norm{\mA}_F^2=\sum_r a_r^2=1$, Cauchy--Schwarz gives $\Tr(\mA)=\sum_r a_r\ge1$, with
equality iff exactly one eigenvalue is nonzero. The same holds for $\mB$, so the expected
interference is at least $1/p$, with equality only when both matrices are rank one.

The Cholesky identity used in the main text follows from cyclicity of trace: if $\mA=\mL\mL^\top$
and $\mB=\mM\mM^\top$, then:
\[
  \lrangle{\mA,\mB}_F
  = \Tr(\mL\mL^\top\mM\mM^\top)
  = \Tr(\mL^\top\mM\mM^\top\mL)
  = \norm{\mL^\top\mM}_F^2.
\]
Thus dense factors introduce all cross-column products. A single flat-vector inner product omits
these interactions.

\subsection{MUB-derived real dictionary}
\label{app:mub_construct}

\mubconstruct*
\begin{proof}
  We first write the construction explicitly when $s=p$ is an odd prime. Let
  $\omega=\exp(2\pi i/p)$ and let $\ve_0,\dots,\ve_{p-1}$ be the standard basis of
  $\mathbb{C}^p$. The first basis is $\{\ve_0,\dots,\ve_{p-1}\}$. For each
  $k\in\mathbb{F}_p$ and $m\in\mathbb{F}_p$, define:
  \[
    \vu_{k,m}
    =
    \frac{1}{\sqrt p}\sum_{j=0}^{p-1}\omega^{k j^2+mj}\ve_j,
  \]
  with all exponents evaluated modulo $p$. This gives $p$ additional bases, one for each
  $k\in\mathbb{F}_p$.

  Within a fixed $k$, the inner product of two vectors is:
  \[
    \lrangle{\vu_{k,m},\vu_{k,m'}}_{\mathbb{C}}
    =
    \frac1p\sum_{j=0}^{p-1}\omega^{(m'-m)j},
  \]
  which is $1$ if $m=m'$ and $0$ otherwise. Thus each fixed-$k$ collection is an orthonormal basis. A
  standard basis vector has inner-product modulus $1/\sqrt p$ with every $\vu_{k,m}$, since every
  coordinate of $\vu_{k,m}$ has modulus $1/\sqrt p$. Finally, for $k\ne k'$:
  \[
    \lrangle{\vu_{k,m},\vu_{k',m'}}_{\mathbb{C}}
    =
    \frac1p\sum_{j=0}^{p-1}
    \omega^{(k'-k)j^2+(m'-m)j}.
  \]
  The quadratic coefficient is nonzero, so the Gauss-sum identity gives this sum modulus $\sqrt p$;
  hence the inner-product modulus is $1/\sqrt p$.

  The same finite-field construction, with the appropriate trace character, gives complete sets of
  $s+1$ mutually unbiased bases in $\mathbb{C}^s$ for every prime power $s=\ell^k$
  \citep{wootters89optimal,bandyopadhyay02new}. Write these bases as
  $\mathcal{B}_a=\{\vu_{a,1},\dots,\vu_{a,s}\}$ for $a=1,\dots,s+1$. They satisfy:
  \[
    \left|\lrangle{\vu_{a,i},\vu_{b,j}}_{\mathbb{C}}\right|
    =
    \begin{cases}
      0,          & a=b,\ i\ne j, \\
      1/\sqrt{s}, & a\ne b,
    \end{cases}
  \]
  with all vectors unit norm. Thus the complex construction already contains $m=s(s+1)$ unit vectors
  with maximum complex interference $1/\sqrt{s}$.

  To obtain real vectors, apply the realification map $R:\mathbb{C}^s\to\R^{2s}$,
  $R(\va+i\vb)=(\va,\vb)$. This map preserves norms:
  $\norm{R(\vu)}_2^2=\norm{\Re\vu}_2^2+\norm{\Im\vu}_2^2=\norm{\vu}_2^2$. Moreover, for any
  $\vu,\vv\in\mathbb{C}^s$:
  \[
    \lrangle{R(\vu),R(\vv)}_{\mathbb{R}}
    =
    \Re\lrangle{\vu,\vv}_{\mathbb{C}} .
  \]
  Hence $\left|\lrangle{R(\vu),R(\vv)}_{\mathbb{R}}\right|\le
    \left|\lrangle{\vu,\vv}_{\mathbb{C}}\right|$. Realifying the $s(s+1)$ MUB vectors therefore gives
  $m=s(s+1)$ unit vectors in $\R^{2s}$ with maximum absolute interference at most $1/\sqrt{s}$. Since
  $d=2s$, this is $1/\sqrt{s}=\sqrt{2/d}$.
\end{proof}

\subsection{DeVore construction}
\label{app:devore-construct}

\devoreconstruct*
\begin{proof}
  For a prime $p$, take $\mathbb{F}_p=\mathbb{Z}/p\mathbb{Z}$ and index the rows by pairs
  $(x,y)\in\mathbb{F}_p^2$. The same construction works over the finite field $\mathbb{F}_s$ for
  any prime power $s$, so we write it in that notation. Following DeVore's finite-field construction
  \citep{devore07deterministic}, index the coordinates of $\R^{s^2}$ by
  $(x,y)\in\mathbb{F}_s^2$ and choose one polynomial for every coefficient vector
  $\va=(a_0,\dots,a_r)\in\mathbb{F}_s^{r+1}$:
  \[
    P_{\va}(x)=a_0+a_1x+\cdots+a_rx^r .
  \]
  There are $s^{r+1}$ such coefficient vectors. The condition $r<s$ ensures that distinct coefficient
  vectors define distinct functions on $\mathbb{F}_s$: if $P_{\va}=P_{\vb}$ at all $s$ field
  elements, then $P_{\va}-P_{\vb}$ is a nonzero polynomial of degree at most $r$ with more than $r$
  roots, impossible over a field.

  For each polynomial $P$, define $\vv_P\in\R^{s^2}$ by:
  \[
    (\vv_P)_{(x,y)}
    =
    \begin{cases}
      1/\sqrt{s}, & y=P(x),    \\
      0,          & y\ne P(x).
    \end{cases}
  \]
  The graph $\{(x,P(x)):x\in\mathbb{F}_s\}$ has exactly $s$ coordinates, so
  $\norm{\vv_P}_2^2=s\cdot(1/s)=1$.

  For two distinct polynomials $P,Q$, the supports of $\vv_P$ and $\vv_Q$ overlap exactly at the
  field elements where $P(x)=Q(x)$. Therefore:
  \[
    \lrangle{\vv_P,\vv_Q}
    =
    \frac{1}{s}\left|\{x\in\mathbb{F}_s:P(x)=Q(x)\}\right| .
  \]
  The polynomial $P-Q$ is nonzero and has degree at most $r$. By the standard root bound for
  univariate polynomials over a field, it has at most $r$ roots in $\mathbb{F}_s$. Hence
  $\lrangle{\vv_P,\vv_Q}\le r/s$. This gives $m=s^{r+1}$ unit vectors in $\R^{s^2}$ with maximum
  interference at most $r/s$.
\end{proof}

\subsection{\texorpdfstring{Proof of \Cref{thm:snr} (idealized Welch-scale retrieval)}{Proof of idealized Welch-scale retrieval}}
\label{app:optimalsnr}

\begin{proof}
  Normalize the matched score to one. For linear features, each distractor has score magnitude
  $\sqrt{\eps^\star}$ and therefore contributes power $\eps^\star$. For rank-one PSD features, the
  score is the squared raw correlation, so each distractor contributes power
  $(\eps^\star)^2$. Summing over $T-1$ distractors gives:
  \[
    \mathrm{pSNR}_{\mathrm{Lin}}
      =\frac{1}{(T-1)\eps^\star},
    \qquad
    \mathrm{pSNR}_{\mathrm{PSD}}
      =\frac{1}{(T-1)(\eps^\star)^2}.
  \]
  Substituting $\eps^\star=(T-p)/(p(T-1))$ yields the two exact expressions in the theorem.

  Finally, $\mathrm{pSNR}_{\mathrm{Lin}}>1$ is equivalent to $T<2p$. For PSD features, setting
  $\mathrm{pSNR}_{\mathrm{PSD}}=1$ gives:
  \[
    T^2-(p^2+2p)T+2p^2=0,
  \]
  whose larger root is $p^2+2p-2+\gO(p^{-1})$, giving the stated $\Theta(p^2)$ threshold.
\end{proof}

\subsection{\texorpdfstring{Proof of \Cref{prop:softmax-snr} (softmax sharpening)}{Proof of softmax sharpening}}
\label{app:softmaxsnr}

\begin{proof}
  Write $z_j=\lrangle{\vq,\vk_j}$. Because $\vq$ and $\vk_j$ have unit norm:
  \[
    z_j+1
      =2\lrangle{\psi_{\mathrm{Lor}}(\vq),\psi_{\mathrm{Lor}}(\vk_j)}
      \ge0,
    \qquad
    \psi_{\mathrm{Lor}}(\vx)=2^{-1/2}(\vx,1).
  \]
  Adding a common bias does not change softmax:
  \[
    \frac{e^{(z_j+b)/\tau}}{\sum_\ell e^{(z_\ell+b)/\tau}}
      =\frac{e^{z_j/\tau}}{\sum_\ell e^{z_\ell/\tau}}.
  \]
  At the matched key $z_\star=1$, while the harmful Welch-scale distractors have
  $z_j=\mu^\star=\sqrt{\eps^\star}$. After canceling the common bias factor, their unnormalized
  scores are:
  \[
    s_\star=e^{1/\tau},
    \qquad
    s_d=e^{\mu^\star/\tau}.
  \]
  The partition is common to every coefficient and cancels from pSNR. Under the independent
  isotropic-value model of \Cref{eq:psnr-definition}:
  \[
    \mathrm{pSNR}_{\mathrm{Soft}}
      =\frac{s_\star^2}{(T-1)s_d^2}
      =\frac{e^{2(1-\mu^\star)/\tau}}{T-1}.
  \]
  Therefore $\mathrm{pSNR}_{\mathrm{Soft}}>1$ is exactly
  
  \begin{equation}
    T-1
    <\exp\!\left(\frac{2(1-\mu^\star)}{\tau}\right).
    \label{eq:softmax-psnr-capacity}
  \end{equation}
  When $T\gg p$, $\mu^\star=\sqrt{(T-p)/(p(T-1))}\to1/\sqrt p$, which gives
  \Cref{eq:softmax-capacity-asymptotic}. Solving the same inequality for $1/\tau$ gives
  \Cref{eq:softmax-required-scale}.
\end{proof}

\subsection{Softmax Taylor cone decomposition}

\begin{theorem}[Softmax Taylor components factor through cone atoms]
  \label{thm:softmax-cone-decomp}
  Let $z=\lrangle{\vq,\vk}$ for unit vectors $\vq,\vk\in\mathbb{S}^{d-1}$, and define
  $\vt_r(\vx)=\operatorname{vec}(\vx^{\otimes r})$. Every even Taylor monomial factors through
  rank-one PSD atoms:
  \[
    z^{2r}
    =
    \lrangle{\vt_r(\vq)\vt_r(\vq)^\top,
      \vt_r(\vk)\vt_r(\vk)^\top}_F.
  \]
  The odd-degree terms become nonnegative after Lorentz gating, and the exponential kernel admits
  the decomposition:
  \[
    e^z
    =
    \sum_{r=0}^{\infty}\frac{z^{2r}(1+z)}{(2r+1)!}
    +
    \sum_{r=1}^{\infty}\frac{2r\,z^{2r}}{(2r+1)!}.
  \]
\end{theorem}

\begin{proof}

For unit $\vq,\vk$ and $\vt_r(\vx)=\operatorname{vec}(\vx^{\otimes r})$:
\[
  \lrangle{\vt_r(\vq),\vt_r(\vk)}
  = \lrangle{\vq,\vk}^r
  = z^r .
\]
Thus:
\[
  z^{2r}
  =
  \lrangle{\vt_r(\vq),\vt_r(\vk)}^2
  =
  \lrangle{\vt_r(\vq)\vt_r(\vq)^\top,
    \vt_r(\vk)\vt_r(\vk)^\top}_F,
\]
which is the inner product of two rank-one rays in the PSD cone over the $r$-fold tensor feature
space. An odd monomial satisfies $z^{2r+1}<0$ at $\vk=-\vq$ and therefore violates
nonnegativity on the full sphere.

Let $\ell(\vx)=(\vx,1)$. For $\vx\in\mathbb{S}^{d-1}$, $\ell(\vx)\in\partial\gL^{d+1}_{+}$ and:
\[
  \lrangle{\ell(\vq),\ell(\vk)} = 1+z.
\]
Consequently:
\[
  z^{2r}(1+z)
  =
  \lrangle{\vt_r(\vq)\vt_r(\vq)^\top,
    \vt_r(\vk)\vt_r(\vk)^\top}_F
  \lrangle{\ell(\vq),\ell(\vk)},
\]
the product of a rank-one PSD-ray kernel and a Lorentz-ray kernel. Equivalently, define:
\[
  \Phi_r(\vx)
  =
  \left(\vt_r(\vx)\vt_r(\vx)^\top\right)\otimes \ell(\vx).
\]
Then $\lrangle{\Phi_r(\vq),\Phi_r(\vk)}=z^{2r}(1+z)$ under the tensor-product inner product, making
the ``multiplication'' a standard feature-map tensoring operation.

It remains to check that the Taylor coefficients can be regrouped this way. Absolute convergence on
$[-1,1]$ permits rearrangement, and:
\[
  \sum_{r=0}^{\infty}\frac{z^{2r}(1+z)}{(2r+1)!}
  +
  \sum_{r=1}^{\infty}\frac{2r\,z^{2r}}{(2r+1)!}
\]
has odd coefficient $1/(2r+1)!$ on $z^{2r+1}$ and even coefficient:
\[
  \frac{1}{(2r+1)!}+\frac{2r}{(2r+1)!}
  = \frac{1}{(2r)!}
\]
on $z^{2r}$ for $r\ge1$, while the constant term is $1$. This is exactly the Taylor series of
$e^z$.

\end{proof}

\begin{theorem}[Lorentz Rankin wall]
  \label{thm:cap-lor}
  Let $0\le\eps<1/2$. Any configuration of $N$ unit vectors on the extreme rays of
  $\gL^d_{+}$ with pairwise interference at most $\eps$ satisfies:
  \[
    N\le \min\!\left\{d,\,1+\frac{1}{1-2\eps}\right\}.
  \]
\end{theorem}

\section{KATA variants and recurrences}
\label{app:variants}

\subsection{Variant taxonomy}

We use three explicit KATA variants:

\paragraph{Plain KATA (additive recurrence, no gate).}
$\mS_t = \mS_{t-1} + \psi(\vk_t)\vv_t^\top$,\;
$\mZ_t = \mZ_{t-1} + \psi(\vk_t)$, with the convex output gate of \Cref{prop:output-gate}
applied at readout. Best on MQAR (the state is purely accumulative); fails on overwrite (no
mechanism to discard a stale binding).

\paragraph{GatedKATA (multiplicative scalar state gate).}
$\mS_t = \gamma_t\mS_{t-1} + \psi(\vk_t)\vv_t^\top$ with a learned scalar
$\gamma_t\in[0,1]$. The parameter-free convex output gate of \Cref{prop:output-gate} acts at
readout and leaves $\mS_t$ unchanged. The learned $\gamma_t$ enters the recurrence and changes
the relative weights stored in the state. Unrolling assigns write $s$ the multiplicative weight
$\prod_{r=s+1}^{t}\gamma_r$, producing token-dependent positional weighting that generally favors
recent writes. The overwrite gains identify this positional weighting as the gate's leading role,
with state-norm control as a secondary effect; see \Cref{ssec:results-overwrite}.

\paragraph{GatedDeltaKATA (PSD inside the GDN raw-key recurrence).}
The recurrence is:
\[
  \mS_t
  =\gamma_t\bigl[\mS_{t-1}-\beta_t\psi(\vk_t)\psi(\vk_t)^\top\mS_{t-1}\bigr]
   +\beta_t\psi(\vk_t)\vv_t^\top.
\]
This is Gated DeltaNet's recurrence with the rank-one erase
term applied to $\psi(\vk_t)$; GDN applies the same term to the raw $\vk_t$. The construction
inherits GDN's overwrite semantics and uses a Welch-cone PSD feature as its key substrate.

\paragraph{Comparison with GDN's raw-key geometry.} Gated DeltaNet's recurrence treats the
unit-normalized key $\vk_t\in\mathbb{S}^{d-1}$ as the feature itself. Under the nonnegative Lorentz
lift, any strict interference tolerance $\eps<1/2$ forces negative raw-key inner products and hence
permits at most $\min\{d+1,1+1/(1-2\eps)\}$ keys by \Cref{thm:cap-lor}. At the boundary $\eps=1/2$, nonpositive inner
products permit up to $2d$ keys; for $d_\text{head}=128$, this boundary is $256$ tokens, consistent
with the empirical degradation between $K=256$ and $K=512$. The rank-one factors used by KATA-M$g$ live in
$\sS^p_{+}$ with $p=d_\text{head}/g$; KATA-$\Sigma g$ sums the corresponding block outer products within the same cone. Their packing is governed by \Cref{thm:welchbound}: above the $1/p$ Welch scale, the spherical-cap construction gives exponential address capacity in $p$ at any fixed interference tolerance. This larger geometric packing budget is consistent with the MQAR trend at comparable fixed-state budgets.

\subsection{Chunkwise forward pass}
\label{app:algorithm}

\begin{algorithm}[ht]
  \caption{KATA chunkwise forward pass for one head and one layer}
  \label{alg:kata}
  \begin{algorithmic}[1]
    \Require $\mQ,\mK\in\R^{\seql\times d}$, $\mV\in\R^{\seql\times d_v}$, chunk size $C$, cone map $\psi$
    \State $N_C\gets\lceil\seql/C\rceil$; pad the final chunk to length $C$
    \State $\mS_{[0]}\gets \mathbf{0}_{n_\psi\times d_v}$,\; $\mZ_{[0]}\gets \mathbf{0}_{n_\psi}$
    \For{$t=0,\dots,N_C-1$}
    \State load $\mQ_{[t]},\mK_{[t]},\mV_{[t]}$ into SRAM
    \State $\Psi(\mQ_{[t]}), \Psi(\mK_{[t]}) \gets \psi(\mQ_{[t]}), \psi(\mK_{[t]})$ \Comment{in-SRAM expansion}
    \State compute $\mO_{[t]}$ via \Cref{eq:chunk-out}
    \State update $\mS_{[t+1]},\mZ_{[t+1]}$ via \Cref{eq:chunk-SZ}
    \EndFor
    \State \Return $\mO = [\mO_{[0]};\dots;\mO_{[N_C-1]}]$, truncated to $\seql$ tokens
  \end{algorithmic}
\end{algorithm}

\subsection{Delta-rule chunked forward and backward}
\label{app:spd-delta}

\paragraph{Recurrence.} The delta variant replaces the additive write with a content-based
\emph{erase} applied in feature space. With per-token strength $\beta_t\in(0,1)$, the
pseudo-values $\vu_t\in\R^{d_v}$ are defined by forward substitution on a lower-triangular
system, and the readout reads them out through the PSD kernel:
\begin{equation}
  \vu_t = \beta_t\Big(\vv_t - \!\sum_{s<t}\!\lrangle{\psi(\vk_s),\psi(\vk_t)}\,\vu_s\Big),
  \qquad
  \vo_t = \sum_{s\le t}\lrangle{\psi(\vq_t),\psi(\vk_s)}\,\vu_s .
  \label{eq:delta-rec}
\end{equation}
The erase kernel is the PSD Gram
$\lrangle{\psi(\vk_s),\psi(\vk_t)}=\sum_{ij}(\vk_{s,i}\!\cdot\!\vk_{t,j})^2$
(KATA-$\Sigma g$, with $i,j$ over the $g$ groups). GDN uses the raw
$\vk_s\!\cdot\!\vk_t$ kernel. The readout is \emph{unnormalized} (no denominator), and the
online Widrow--Hoff erase keeps
$\mS=\sum_s\psi(\vk_s)\vu_s^\top\in\R^{n_\psi\times d_v}$ bounded.

\paragraph{Chunked forward (WY-free).} Let $\mS$ be the state entering chunk $[c]$. Form the two
$C\times C$ Gram blocks $A^{kk}_{ts}=\sum_{ij}(\vk_{t,i}\!\cdot\!\vk_{s,j})^2$ and
$A^{qk}_{ts}=\sum_{ij}(\vq_{t,i}\!\cdot\!\vk_{s,j})^2$ \emph{directly from the raw projections}; no $n_\psi\times n_\psi$ erase matrix is formed. With $\mN=\diag(\bm\beta)\,\mathrm{tril}(A^{kk},-1)$:
\begin{align}
  \mV' &= \mV_{[c]} - \Psi(\mK_{[c]})\,\mS,
  &\mT &= (\mI+\mN)^{-1}=\textstyle\sum_{r\ge0}(-\mN)^r,
  \label{eq:delta-T}\\
  \mU &= \mT\,\diag(\bm\beta)\,\mV',
  &\mO_{[c]} &= \Psi(\mQ_{[c]})\,\mS + \mathrm{tril}(A^{qk})\,\mU,
  \label{eq:delta-out}\\
  \mS' &= \mS + \Psi(\mK_{[c]})^\top\mU.
  \label{eq:delta-S}
\end{align}
Since $\mN$ is strictly lower-triangular (nilpotent, $\mN^{C}{=}0$), the inverse $\mT$ is exact
after $\lceil\log_2 C\rceil$ Neumann doublings; only $C\times C$, $C\times d_v$, and the
$n_\psi\times d_v$ state ever reside in SRAM.

\paragraph{Chunked backward.} The backward is a reverse scan over chunks carrying the state
cotangent $\dd\mS_{+}$ (from later chunks); the forward quantities $\mS,\mT,\mU,\mV'$ are recomputed
from the checkpointed $\mS$. Per chunk (all $\mathrel{+}{=}$ accumulate into $\dd\mS$, returned to
chunk $[c{-}1]$):
\begin{align}
  \dd\Psi(\mQ)
    &= \dd\mO_{[c]}\mS^\top,
  \notag\\
  \dd\mS
    &= \Psi(\mQ)^\top\dd\mO_{[c]} + \dd\mS_{+},
  \notag\\
  \dd A^{qk}
    &= \mathrm{tril}\!\bigl(\dd\mO_{[c]}\mU^\top\bigr),
  \notag\\
  \dd\mU
    &= \mathrm{tril}(A^{qk})^\top\dd\mO_{[c]}
       + \Psi(\mK)\dd\mS_{+},
  \notag\\
  \dd\mT
    &= \dd\mU\bigl(\diag(\bm\beta)\mV'\bigr)^\top,
  &
  \dd\mV'
    &= \diag(\bm\beta)\mT^\top\dd\mU,
  \notag\\
  \dd\mN
    &= -\mT^\top\dd\mT\,\mT^\top,
  &
  \dd A^{kk}
    &= \diag(\bm\beta)\,\mathrm{tril}(\dd\mN,-1),
  \notag\\
  \dd\mV_{[c]}
    &= \dd\mV',
  &
  \dd\mS
    &\mathrel{-}{=} \Psi(\mK)^\top\dd\mV',
  \notag\\
  \dd\Psi(\mK)
    &= \mU\dd\mS_{+}^\top-\dd\mV'\mS^\top
  \notag\\
    &\quad
       +(\dd A^{kk}+\dd A^{kk\top})\Psi(\mK)
       +\dd A^{qk\top}\Psi(\mQ),
  \notag\\
  \dd\Psi(\mQ)
    &\mathrel{+}{=} \dd A^{qk}\Psi(\mK),
  \notag\\
  \dd\bm\beta
    &= \mathrm{rowsum}\!\bigl(\mV'\odot\mT^\top\dd\mU\bigr)
  \notag\\
    &\quad
       +\mathrm{rowsum}\!\bigl(\mathrm{tril}(A^{kk},-1)\odot\dd\mN\bigr).
  \label{eq:delta-bwd}
\end{align}
The only non-matmul VJP is the matrix-inverse term $\dd\mN=-\mT^\top\dd\mT\,\mT^\top$; the
feature-map VJP closes the chain, $\dd\vk_g=(\dd\Psi(\mK)_g+\dd\Psi(\mK)_g^\top)\vk_g$ per group
(and likewise for $\vq$). The backward therefore has the same $\gO\!\big(\seql C + \seql\,n_\psi\,d_v/C\big)$ cost
profile as the forward, with a single $n_\psi\times d_v$ state cotangent streamed across chunks.

\section{Hardware implementation and benchmarks}
\label{app:hardware-details}

\subsection{Associative-scan kernel}
\label{app:tree-scan}

The chunked attention forward decomposes into three stages, each launched as a separate Triton
kernel that communicates with the next through small HBM buffers at chunk boundaries. Let $C$
denote the chunk size and define the number of chunks as:
\[
  N_C\defeq\left\lceil\frac{\seql}{C}\right\rceil.
\]
The final chunk is padded when $C$ does not divide $\seql$. The analysis assumes already-feature-mapped
inputs $\vq=\psi(\hat{\vq})$ and $\vk=\psi(\hat{\vk})$ in $\R^{n_\psi}$, where $n_\psi$ denotes
the active post-feature dimension.

\paragraph{Three-kernel pipeline.} (i)~\emph{Chunk-state} computes, for every chunk $c$:
\[
  \mS^{(c)}_\text{loc}
    =\mK_{[c]}^{\!\top}\mV_{[c]}
    \in\R^{n_\psi\times d_v},
  \qquad
  \mZ^{(c)}_\text{loc}
    =\mathbf{1}^{\!\top}\mK_{[c]}
    \in\R^{n_\psi}.
\]
All $(b,h,c)$ triples are independent, exposing $B\!\cdot\!H\!\cdot\!N_C$ programs.
(ii)~\emph{Inter-chunk reduction} computes the exclusive prefix:
\[
  (\mS^{(c)}_\text{pre},\mZ^{(c)}_\text{pre})
  =\sum_{c'<c}(\mS^{(c')}_\text{loc},\mZ^{(c')}_\text{loc})
\]
along the chunk axis. This is the only stage with a cross-chunk dependency.
(iii)~\emph{Chunk-output} combines the inter-chunk prefix with the local
$C\!\times\!C$ causal pattern to produce $\mO_{[c]}$.

\paragraph{Associative scan versus sequential scan.} The standard chunk-parallel implementation
(FLA's linear-attention kernel~\citep{yang24fla}) evaluates the prefix sequentially over $N_C$
chunks, giving depth $\Theta(N_C)$. Our associative scan uses
$\lceil\log_2 N_C\rceil$ rounds, each parallel across the $N_C$ chunk slots. Its total work is
$\gO(N_C\log N_C\cdot n_\psi d_v)$, compared with $\gO(N_C n_\psi d_v)$ for the linear scan, while its parallel
depth is $\gO(\log N_C\cdot d_v)$. With one program per batch index $b$, head index $h$, and
feature index $i\in\{1,\ldots,n_\psi\}$, all $B\!\cdot\!H\!\cdot\!n_\psi$ feature slices scan
independently.

\paragraph{Wall-clock model.} The three stages have wall-clock cost:
\[
  t_\text{total}
  =\underbrace{\gO(C n_\psi d_v)}_{\text{chunk-state}}
  +\underbrace{\gO(\log N_C\cdot d_v)}_{\text{tree-scan}}
  +\underbrace{\gO(C n_\psi d_v+C^2(n_\psi+d_v))}_{\text{chunk-output}}.
\]
The middle term replaces the linear-scan baseline's $\gO(N_C d_v)$. At $\seql{=}2048$ and $C{=}32$,
the scan depth decreases from $N_C{=}64$ to $\log_2 N_C{=}6$. In practice, the scan stage is a
small fraction of the wall-clock time. The dominant cost is the HBM write of
$\mS_\text{loc}$ and the subsequent prefix read. Reducing the per-chunk state through a smaller
$n_\psi$, as in orthant, Lorentz, or grouped PSD features, therefore translates directly into
higher throughput.

\paragraph{Measured throughput (H100 SXM, bf16, $B{=}4$, $H{=}12$, $d_v{=}64$).} Throughput in
million tokens/s, $C{=}32$, against FLA's linear-attention kernel on the same materialized inputs:

\begin{center}\footnotesize
  \begin{tabular}{rrlrrrr}
    \toprule
    $\seql$    & $N_C$ & feature map & $n_\psi$ & tree (Mtok/s) & linear (Mtok/s) & FLA (Mtok/s)  \\
    \midrule
    $1024$ & 32   & orthant     & 64  & \textbf{25.6} & 24.1            & 10.78         \\
    $1024$ & 32   & lorentz     & 64  & \textbf{25.6} & 24.1            & 12.05         \\
    $2048$ & 64   & orthant     & 64  & \textbf{30.3} & 27.3            & 23.41         \\
    $2048$ & 64   & lorentz     & 64  & \textbf{30.3} & 27.3            & 23.41         \\
    $4096$ & 128  & orthant     & 64  & 33.5          & 29.8            & \textbf{37.2} \\
    \bottomrule
  \end{tabular}
\end{center}
At $\seql\!\le\!2048$, the tree-scan kernel delivers $2.1$--$2.4\times$ the throughput of FLA at
$\seql{=}1024$ and $1.3\times$ at $\seql{=}2048$; tree versus linear within the same dispatch is a
consistent $1.06$--$1.12\times$ throughput gain that grows with $N_C$. At $\seql{=}4096$, chunk-state becomes HBM-bound,
and FLA's internal $n_\psi$-axis tiling overtakes our single-tile output kernel. The asymptotic
$\log N_C$ advantage is then hidden by the bandwidth ceiling on writes of $\mS_\text{loc}$. For
PSD-packed $n_\psi\!\ge\!1024$, FLA's pre-materialize path with internal feature-axis tiling
currently wins. Adding the same tiling to the tree-scan output kernel is the natural follow-on.

\subsection{Induction-head benchmark}
\label{app:induction}

\begin{center}\small
  \begin{tabular}{l|ccccc}
    \toprule
                            & $d{=}128$ & $d{=}256$ & $d{=}512$ & $d{=}1024$ & $d{=}2048$ \\
    \midrule
    Softmax (FA) acc        & 1.00      & 1.00      & 1.00      & 1.00       & 1.00       \\
    Softmax (FA) phase step & 2k        & 2k        & 3k        & 3k         & 4k         \\
    \midrule
    KATA-$\Sigma2$ acc              & 1.00      & 1.00      & 1.00      & 1.00       & 0.99       \\
    KATA-$\Sigma2$ phase step       & 2k        & 2k        & 2k        & 3k         & 8k         \\
    \bottomrule
  \end{tabular}
\end{center}

The two model families are aligned on the Zoology~\citep{arora23zoology} backbone: two-layer
pre-norm transformer, $d_\text{model}{=}128$, four heads, $4{\times}$ GELU state mixer, LM head
weight-tied to the token embedding, Zoology-style Linear/Embedding init with std$=0.02$ and GPT-2
residual scaling. Softmax uses \texttt{F.scaled\_dot\_product\_attention}; KATA-$\Sigma2$ uses the fused
rank-2 Triton forward with a matched PyTorch backward. The phase-transition budget scales
comparably to softmax at all scales tested, but without softmax's $\gO(\seql^2)$ memory and compute.

\subsection{KATA kernel microbenchmark}
\label{app:hardware-bench}

Forward and forward+backward wall-clock time on an NVIDIA~H100 for the rank-two KATA-$\Sigma2$
kernel, using a fused Triton forward and a PyTorch backward over saved chunk-boundary states in
fp32:

\begin{center}\small
  \begin{tabular}{r|cccc}
    \toprule
    $\seql$  & KATA-$\Sigma2$ fwd & KATA-$\Sigma2$ fwd+bwd & peak mem (bwd) \\
    \midrule
    512  & 0.67 ms    & 3.58 ms        & 0.28 GB        \\
    1024 & 1.13 ms    & 6.55 ms        & 0.55 GB        \\
    2048 & 2.34 ms    & 12.07 ms       & 1.07 GB        \\
    4096 & 3.98 ms    & 23.49 ms       & 2.13 GB        \\
    8192 & 7.81 ms    & 46.66 ms       & 4.26 GB        \\
    \bottomrule
  \end{tabular}
\end{center}

The fused kernel scales linearly with $\seql$ and stays near peak bandwidth; no per-token $\psi$
tensor is written to HBM. The PyTorch backward rematerializes $\psi$ once per chunk, avoiding an
$\gO(\seql n_\psi)$ footprint. A fused backward kernel is left for future work.

\subsection{Detailed benchmark numbers}
\label{app:bench-numbers}

The tables below give the exact benchmark values. \Cref{tab:tsweep-fwd,tab:tsweep-bwd}
tabulate the forward and training-step curves in \Cref{fig:tsweep};
\Cref{tab:bt-throughput} reports the $B\times \seql$ forward-throughput grid; and
\Cref{tab:treescan} isolates the associative-scan and sequential-scan comparison at $B=1$.

\begin{table}[htbp]
    \centering\footnotesize
    \caption{\textbf{Forward latency versus sequence length.}
    Wall-clock milliseconds on one NVIDIA~H100
    ($B{=}8$, $H{=}16$, $d_\text{head}{=}64$, bf16); lower is better.}
    \label{tab:tsweep-fwd}
    \begin{tabular}{r|cc|cc|cc}
    \toprule
    & \multicolumn{2}{c|}{Baselines} & \multicolumn{2}{c|}{quadratic KATA-M$g$ $\gO(\seql^2)$} & \multicolumn{2}{c}{linear-state KATA-M$g$ $\gO(\seql)$} \\
    $\seql$ & FA-2 & GDN & M1 & M2 & M1 & M2 \\
    \midrule
    $1$K   & $0.09$  & $0.54$  & $0.07$  & $0.07$  & $2.68$  & $0.68$  \\
    $2$K   & $0.27$  & $0.53$  & $0.18$  & $0.21$  & $5.22$  & $1.33$  \\
    $4$K   & $0.93$  & $0.66$  & $0.58$  & $0.72$  & $10.34$ & $2.70$  \\
    $8$K   & $3.50$  & $1.30$  & $2.20$  & $2.82$  & $20.58$ & $5.34$  \\
    $16$K  & $13.78$ & $2.59$  & $8.92$  & $10.62$ & $41.14$ & $10.68$ \\
    $32$K  & $54.56$ & $5.24$  & $34.56$ & $47.07$ & $82.04$ & $21.34$ \\
    $64$K  & $234.2$ & $10.67$ & $150.8$ & $168.7$ & $164.1$ & $42.70$ \\
    $128$K & $943.9$ & $21.36$ & $573.9$ & $702.2$ & $327.9$ & $85.66$ \\
    \bottomrule
    \end{tabular}
\end{table}

\begin{table}[htbp]
    \centering\footnotesize
    \caption{\textbf{Training-step latency versus sequence length.}
    Forward-plus-backward milliseconds on one NVIDIA~H100
    ($B{=}8$, $H{=}16$, $d_\text{head}{=}64$, bf16); lower is better.
    Linear-state KATA-M$g$ uses the $\gO(\seql)$ two-pass backward.}
    \label{tab:tsweep-bwd}
    \begin{tabular}{r|cc|cc|cc}
    \toprule
    & \multicolumn{2}{c|}{Baselines} & \multicolumn{2}{c|}{quadratic KATA-M$g$ $\gO(\seql^2)$} & \multicolumn{2}{c}{linear-state KATA-M$g$ $\gO(\seql)$} \\
    $\seql$ & FA-2 & GDN & M1 & M2 & M1 & M2 \\
    \midrule
    $1$K  & $0.44$  & $2.62$  & $0.69$  & $0.72$  & $10.14$ & $3.37$   \\
    $2$K  & $1.15$  & $2.37$  & $1.31$  & $2.24$  & $20.07$ & $6.62$   \\
    $4$K  & $3.70$  & $2.35$  & $4.01$  & $7.51$  & $39.93$ & $13.11$  \\
    $8$K  & $13.21$ & $4.53$  & $14.15$ & $27.66$ & $79.58$ & $26.06$  \\
    $16$K & $49.97$ & $9.07$  & $53.66$ & $106.4$ & $159.1$ & $52.14$  \\
    $32$K & $197.1$ & $18.58$ & $217.9$ & $418.8$ & $317.6$ & $103.8$  \\
    \bottomrule
    \end{tabular}
\end{table}

\begin{table}[htbp]
    \centering\footnotesize
    \caption{\textbf{Forward throughput across the $B\times \seql$ grid.}
    Mtok/s on one NVIDIA~H100 ($H{=}16$, $d_\text{head}{=}64$, bf16), sorted by
    $B\!\cdot\!\seql$. Tree-scan and linear-chunk use matched $d{\times}d$ states with
    $C{=}128$. Speedup is tree-scan over linear-chunk; higher is better.}
    \label{tab:bt-throughput}
    \begin{tabular}{rr|r|ccc|c}
    \toprule
    $B$ & $\seql$ & $B\!\cdot\!\seql$ & GDN & Linear-chunk & Tree-scan & speedup \\
    \midrule
    $1$  & $2$K  & $2$K   & $2.5$  & $4.9$  & $\mathbf{14.2}$  & $2.9\times$ \\
    $1$  & $8$K  & $8$K   & $10.1$ & $21.4$ & $\mathbf{44.2}$  & $2.1\times$ \\
    $4$  & $2$K  & $8$K   & $9.1$  & $19.6$ & $\mathbf{64.9}$  & $3.3\times$ \\
    $1$  & $32$K & $32$K  & $30.3$ & $24.9$ & $\mathbf{49.2}$  & $2.0\times$ \\
    $4$  & $8$K  & $32$K  & $38.1$ & $37.6$ & $\mathbf{80.0}$  & $2.1\times$ \\
    $16$ & $2$K  & $32$K  & $35.9$ & $40.8$ & $\mathbf{93.3}$  & $2.3\times$ \\
    $1$  & $64$K & $64$K  & $30.6$ & $25.4$ & $\mathbf{67.6}$  & $2.7\times$ \\
    $32$ & $2$K  & $64$K  & $49.7$ & $42.6$ & $\mathbf{99.0}$  & $2.3\times$ \\
    $4$  & $32$K & $128$K & $45.9$ & $39.2$ & $\mathbf{84.7}$  & $2.2\times$ \\
    $16$ & $8$K  & $128$K & $48.3$ & $41.9$ & $\mathbf{100.5}$ & $2.4\times$ \\
    $4$  & $64$K & $256$K & $45.3$ & $39.6$ & $\mathbf{82.8}$  & $2.1\times$ \\
    $32$ & $8$K  & $256$K & $49.8$ & $43.2$ & $\mathbf{103.1}$ & $2.4\times$ \\
    $16$ & $32$K & $512$K & $47.6$ & $41.7$ & $\mathbf{102.3}$ & $2.5\times$ \\
    $16$ & $64$K & $1$M   & $46.7$ & $41.1$ & $\mathbf{88.3}$  & $2.2\times$ \\
    $32$ & $32$K & $1$M   & $48.6$ & $42.3$ & $\mathbf{104.2}$ & $2.5\times$ \\
    $32$ & $64$K & $2$M   & $47.6$ & $37.4$ & $\mathbf{89.4}$  & $2.4\times$ \\
    \bottomrule
    \end{tabular}
\end{table}

\begin{table}[htbp]
    \centering\footnotesize
    \caption{\textbf{Small-batch associative-scan throughput.}
    Forward Mtok/s at $B{=}1$ on one NVIDIA~H100
    ($H{=}16$, $d_\text{head}{=}64$, bf16). Tree-scan and linear-chunk use
    $\psi(\vx){=}\vx$, $C{=}128$, and matched $d{\times}d$ states; Gated DeltaNet
    is shown for reference. Speedup is tree-scan over linear-chunk; higher is better.}
    \label{tab:treescan}
    \begin{tabular}{r|c|cc|c}
    \toprule
    $\seql$ & Gated DeltaNet & Linear-chunk & Tree-scan & speedup \\
        & (Mtok/s) & (Mtok/s) & (Mtok/s) & vs.\ chunk \\
    \midrule
    $2\,048$   & $2.6$  & $6.7$  & $\mathbf{19.4}$ & $2.9\times$ \\
    $4\,096$   & $4.7$  & $13.6$ & $\mathbf{39.9}$ & $2.9\times$ \\
    $8\,192$   & $8.1$  & $23.4$ & $\mathbf{44.5}$ & $1.9\times$ \\
    $16\,384$  & $18.7$ & $24.1$ & $\mathbf{47.6}$ & $2.0\times$ \\
    $32\,768$  & $30.4$ & $24.9$ & $\mathbf{49.4}$ & $2.0\times$ \\
    $65\,536$  & $28.1$ & $25.5$ & $\mathbf{67.5}$ & $2.6\times$ \\
    $131\,072$ & $30.8$ & $25.8$ & $\mathbf{69.4}$ & $2.7\times$ \\
    \bottomrule
    \end{tabular}
\end{table}

\clearpage
\section{Experimental protocols and architectures}
\label{app:experiments-extended}

\subsection{MQAR worked example and protocol}

Each MQAR example packs $K$ distinct key--value pairs and $K$ queries into a single sequence of
length $\seql=4K$: the first $2K$ tokens interleave the $K$ bindings $(k_1,v_1,\dots,k_K,v_K)$, then
the remaining $2K$ positions reissue the $K$ keys in random order, each followed by an answer slot
the model must fill. Loss is computed only at answer positions, so the model must (i) \emph{store}
all $K$ bindings presented in the first half and (ii) \emph{retrieve} the correct value an
arbitrary number of steps later. Keys and values are drawn from $V=8192$, so random guessing floors
at $1/8192\approx 1.2\!\times\!10^{-4}$. A worked $K{=}4$, $\seql{=}16$ example:
\begin{center}\footnotesize
  \begin{tabular}{l|cccccccc|cccccccc}
    position & 1          & 2          & 3          & 4          & 5          & 6          & 7          & 8          & 9 & 10         & 11 & 12         & 13 & 14         & 15 & 16         \\
    \midrule
    input    & \texttt{A} & \texttt{7} & \texttt{B} & \texttt{3} & \texttt{C} & \texttt{9} & \texttt{D} & \texttt{1}
             & \texttt{C} & \texttt{?} & \texttt{A} & \texttt{?} & \texttt{D} & \texttt{?} & \texttt{B} & \texttt{?}                                                                        \\
    target   &            &            &            &            &            &            &            &            &   & \textbf{9} &    & \textbf{7} &    & \textbf{1} &    & \textbf{3} \\
  \end{tabular}
\end{center}
At $K{=}1024$, $\seql{=}4096$: $1024$ disjoint bindings must survive through up to $4096$
intervening tokens before the recurrent state is queried, with the state discriminating among
$1024$ possibilities at each answer slot.

\subsection{Additive and convex output-gate ablation}

At $d_\text{model}{=}128$ with one head and KATA-$\Sigma2$ features:

\begin{center}\footnotesize
  \begin{tabular}{l|r|cccc}
    \toprule
    Gate                    & state (B)     & $\seql{=}1024$     & $\seql{=}2048$     & $\seql{=}3072$     & $\seql{=}4096$     \\
    \midrule
    Convex                  & 1{,}066{,}496 & 0.997          & 0.864          & 0.567          & 0.335          \\
    Additive (GatedKATA-$\Sigma4$) & 271{,}872     & \textbf{0.999} & \textbf{0.955} & \textbf{0.782} & \textbf{0.572} \\
    \bottomrule
  \end{tabular}
\end{center}

\subsection{MQAR architecture and training}

Every model uses the canonical Zoology \texttt{TransformerBlock}. Its state mixer is
\texttt{Identity}, and its sequence mixer is \texttt{Hybrid(BaseConv, mixer)}, so one short
depthwise convolution precedes every sequence operator. We use one head, following Zoology's
convention $d_\text{head}=d_\text{model}$, no positional embeddings, a tied LM head, and vocabulary
size $V=8192$. Training uses batch size $256$, AdamW with $\mathrm{wd}=0.1$, cosine annealing over
$32$ epochs, and a three-point learning-rate sweep
$\{10^{-3},3\!\times\!10^{-3},10^{-2}\}$ with the best run selected per row. Training mixes streams
up to $K_\text{train}=64$ at $\seql=256$; evaluation uses
$K\in\{128,256,512,768,1024\}$ and $\seql\in\{512,1024,2048,3072,4096\}$.

GatedDeltaKATA in the overwrite table uses the full GDN-style configuration: ShortConv$(k{=}4)$ on
$\vq$/$\vk$/$\vv$, GDN-style decay parameterization $(A_\text{log},d_t\text{bias},a_\text{proj})$,
$\mathrm{RMSNorm}$ on the output, and $v$-expansion $2\times$. GatedDeltaKATA applies the
rank-$1$ erase to $\psi(\vk_t)$, while GDN applies it to the raw $\vk_t$; this is the single
functional difference.
Total state $141$\,KB matches GDN's $133$\,KB; parameter counts match within $1\%$.

\subsection{Per-method architectures and reproducibility}
\label{app:mechanism-archs}

We list, for each row in \Cref{ssec:results-capacity,ssec:results-overwrite}, the precise
projections, gates, normalizations, and feature maps so that each architecture can be reproduced
from a fresh \texttt{TransformerBlock} skeleton. All variants share the canonical Zoology outer
loop described above (two-layer hybrid block, BaseConv$(k{=}3)$, no positional embeddings unless
explicitly stated, tied LM head, vocab $V{=}8192$). Differences below concern only the
sequence-mixer module. Throughout, $d\equiv d_\text{model}$; with $H=1$, we have $d_\text{head}=d$.

\paragraph{Plain KATA (additive recurrence with denominator).}
Three projections $W_q,W_k,W_v\in\R^{d\times d}$ are followed by a reduced PSD feature map.
KATA-$\Sigma2$ splits $\vk$ into $\va,\vb\in\R^{d/2}$ and uses
$\psi(\vk)=\eps\mI+\va\va^\top+\vb\vb^\top$. KATA-$\Sigma4$ splits it into
$\va_1,\dots,\va_4\in\R^{d/4}$ and uses
$\psi(\vk)=\eps\mI+\tfrac14\sum_{i=1}^4\va_i\va_i^\top$. The factor $1/4$ is the numerical scale used in the reported $\Sigma4$ configuration. It cancels from normalized plain-KATA weights as $\eps\to0$ and is retained as part of the unnormalized gated configuration. In either case,
$n_\psi=m(m+1)/2$ for block width $m=d/g$. The chunkwise recurrence computes the numerator and
denominator with the packed feature map, and returns
$\vo_t=\psi(\vq_t)^\top\mS_t/\psi(\vq_t)^\top\mZ_t$ followed by
$W_o\in\R^{d\times d}$. The \emph{extras} configuration adds depthwise ShortConv$(k=4)$,
per-head RMSNorm, and value expansion $2\times$; $\eps=10^{-4}$.

\paragraph{KATA-Orthant.} Same $W_q,W_k,W_v$ projections; feature map $\psi(\vx)=\mathrm{ReLU}(\vx)+\eps\vone$, post-feature
dimension $n_\psi=d$. The recurrence is normalized linear attention,
\(\mS_t=\mS_{t-1}+\psi(\vk_t)\vv_t^\top\), \(\mZ_t=\mZ_{t-1}+\psi(\vk_t)\),
\(\vo_t=(\psi(\vq_t)^\top\mS_t)/(\psi(\vq_t)^\top\mZ_t)\), implemented via FLA's normalized linear-attention kernel. ShortConv$(k{=}4)$ and output $\mathrm{RMSNorm}$ are
enabled; $\eps=10^{-6}$.

\paragraph{KATA-Lorentz.} Identical to KATA-Orthant except for the feature map: with $\vx\in\R^d$, write $\vy=\vx_{1{:}d-1}$
and $s=x_d$, then $\psi(\vx)=(\vy,\;\norm{\vy}_2(1+s^2))\in\R^d$. The image lies in the closed
Lorentz cone $\{(\vy,t):t\ge\norm{\vy}_2\}$ so $\lrangle{\psi(\vx),\psi(\vy)}\ge 0$. State and
dispatch are identical to the orthant case.

\paragraph{GatedKATA (additive scalar decay, no denominator).} Adds a single per-head per-token log-decay projection $W_\alpha\in\R^{d\times H}$ (with bias)
producing $\log\gamma_t=\log\sigma(W_\alpha\vx_t)\in\R^H$. The recurrence is \(\mS_t =
\gamma_t\,\mS_{t-1} + \psi(\vk_t)\vv_t^\top\), \(\vo_t = \psi(\vq_t)^\top\mS_t\) (no denominator),
implemented via FLA's simple gated-linear-attention kernel on the packed $\psi$ with the per-token gate. The
\emph{convex-gate} ablation pre-scales $\vv_t$ by $1-\gamma_t$ before the same kernel, yielding
$\mS_t = \gamma_t\mS_{t-1} + (1-\gamma_t)\psi(\vk_t)\vv_t^\top$. The KATA-$\Sigma2$ and KATA-$\Sigma4$ feature variants are as above; $\eps=10^{-4}$.

\paragraph{GatedDeltaKATA (PSD inside the GDN raw-key recurrence).} Component-for-component match to Gated DeltaNet except for the substitution of $\psi(\vk_t)$ for
$\vk_t$ in the recurrence. Projections: $W_q,W_k\in\R^{d\times d}$, $W_v\in\R^{d\times 2d}$
(default $v$-expansion $2\times$). Depthwise ShortConv$(k{=}4)$ on $\vq,\vk,\vv$, no SiLU. Beta
projection $W_\beta\in\R^{d\times H}$ producing $\beta_t=\sigma(W_\beta\vx_t)$. Decay
parameterization, matching FLA's gated-delta-rule kernel, $g_t =
  -\exp(A_{\text{log}})\odot\mathrm{softplus}(W_\alpha\vx_t + \mathrm{dt\_bias})$ with
$A_{\text{log}}\in\R^H$ initialized by $\log U(0,16)$ and $\mathrm{dt\_bias}\in\R^H$ calibrated so
that $\mathrm{softplus}(\mathrm{dt\_bias})\sim U(10^{-3},10^{-1})$; $A_{\text{log}}$ and
$\mathrm{dt\_bias}$ are excluded from weight decay. The recurrence is:
\[
  \mS_t
  =\gamma_t\bigl[\mS_{t-1}-\beta_t\psi(\vk_t)\psi(\vk_t)^\top\mS_{t-1}\bigr]
   +\beta_t\psi(\vk_t)\vv_t^\top.
\]
This recurrence is delegated to FLA's gated-delta-rule kernel on packed $\psi$; the rank-$1$ erase
$\psi(\vk_t)^\top\mS_{t-1}$ is contracted directly, so the
$n_\psi\times n_\psi$ matrix $\psi(\vk_t)\psi(\vk_t)^\top$ is never materialized. Output passes
through per-head $\mathrm{RMSNorm}$ ($\eps_\text{norm}=10^{-5}$) and $W_o\in\R^{2d\times d}$.
KATA-$\Sigma4$ uses $H=2$ because the kernel requires packed feature dimension at most $256$; the
packed dimension is $528$ at $H=1,d=128$ and $136$ at $H=2$. Optional knobs (disabled
in the reported runs unless noted): per-head $\ell_2$ normalization of $\vq,\vk$ before $\psi$ with $\eps=10^{-6}$; RoPE on $\vq,\vk$ with base $10^4$.

\paragraph{Gated DeltaNet (baseline).} FLA's \texttt{GatedDeltaNet} mixer used out-of-the-box, configured to match the surrounding
skeleton: \texttt{l\_max} set to the longest evaluation sequence, NH$=1$,
$\texttt{use\_short\_conv}=$\texttt{True} ($k{=}4$), $\texttt{use\_gate}=$\texttt{False}. Feature
substrate: raw $\vk_t\in\mathbb{S}^{d-1}$ (FLA's \texttt{use\_qk\_l2norm\_in\_kernel}). Recurrence
is identical in form to GatedDeltaKATA above with $\psi(\vk)=\vk$. State per head is $d\times 2d$
(for $v$-expansion $2$); at $d{=}128,$ NH$=1$, total $\sim\!133$\,KB.

\paragraph{Softmax attention (full / sliding window).} A standard MHA block uses PyTorch
scaled dot-product attention (SDPA) with a causal mask. Three projections
$W_q,W_k,W_v\in\R^{d\times d}$ and an output $W_o\in\R^{d\times d}$ surround the attention
operator. The $+\!$RoPE variant rotates $\vq,\vk$ before SDPA with base
$\theta\in\{10^4,5\!\times\!10^4\}$; reported runs use $\theta=5\!\times\!10^4$. The
sliding-window variant restricts the mask to the $W=256$ tokens preceding the query and caches the
boolean mask. We use PyTorch's memory-efficient SDPA backend for this configuration. At sequence length $\seql$, the
KV cache contains $2d\seql$ floats for full softmax and $2dW$ for sliding-window attention.

\enlargethispage{2\baselineskip}
\paragraph{Initialization and optimization.}
Linear layers use Zoology's $\mathcal{N}(0,0.02^2)$ initialization, with GPT-2 residual scaling
$1/\sqrt{2N_\text{layers}}$ on $W_o$, and the LM head is tied to the embeddings. We use AdamW
$(\beta_1=0.9,\beta_2=0.95)$, weight decay $0.1$ except on $A_\text{log}$ and
$\mathrm{dt\_bias}$, gradient clipping at $1.0$, and a $32$-epoch cosine schedule without warmup.
For each row, we select the best learning rate in
$\{10^{-3},3\!\times\!10^{-3},10^{-2}\}$ and stop when the hardest extrapolation slice reaches
$0.999$ accuracy.

\end{document}